%% file: neurips_2026.tex
\documentclass{article}

\usepackage[preprint]{neurips_2026}

\usepackage[utf8]{inputenc} % allow utf-8 input
\usepackage[T1]{fontenc}    % use 8-bit T1 fonts
\usepackage{hyperref}       % hyperlinks
\usepackage{url}            % simple URL typesetting
\usepackage{booktabs}       % professional-quality tables
\usepackage{amsfonts}       % blackboard math symbols
\usepackage{nicefrac}       % compact symbols for 1/2, etc.
\usepackage{microtype}      % microtypography
\usepackage{xcolor}         % colors
\usepackage{amsmath}
\usepackage{mathtools}
\usepackage{amssymb}
\usepackage{amsthm}
\usepackage{todonotes}
\usepackage{algorithmic}
\usepackage{algorithm}
\usepackage{booktabs}
\usepackage{multirow}
\usepackage[normalem]{ulem}
\useunder{\uline}{\ul}{}
\usepackage{tcolorbox}
\usepackage{enumitem}
\usepackage{caption}

\theoremstyle{plain}
\newtheorem{theorem}{Theorem}[section]

\newtheorem{lemma}[theorem]{Lemma}
\newtheorem{corollary}[theorem]{Corollary}
\theoremstyle{definition}

\newtheorem{assumption}[theorem]{Assumption}
\theoremstyle{remark}
\newtheorem{remark}[theorem]{Remark}

\tcolorboxenvironment{admm}{
  colback=white,
  sharp corners,
  boxrule=0.5pt,
  }

\title{ASTRA: ADMM-Accelerated Topology Reconfiguration for Dynamic Satellite Constellations 
}

\author{%
  Jo\~{a}o Norberto \\
  Department of Computer Science \\
  NOVA School of Science and Technology\\
  Caparica, Portugal\\
  \texttt{jd.norberto@campus.fct.unl.pt} \\
  \And
  Ricardo N. Ferreira \\
  Department of Computer Science \\
  NOVA School of Science and Technology\\
  Caparica, Portugal\\
  \texttt{rjn.ferreira@campus.fct.unl.pt} \\
  \And
  Cl\'{a}udia Soares\\
  Department of Computer Science \\
  NOVA School of Science and Technology\\
  Caparica, Portugal\\
  \texttt{claudia.soares@fct.unl.pt} \\
}

\begin{document}

\maketitle

\begin{abstract}
Dynamic topology reconfiguration is central to the reliability and efficiency of large satellite constellations, yet many existing approaches rely on idealized assumptions such as full constellation deployment or uniform orbital spacing. We present \textbf{Adaptive Satellite Topology via Regret-Aware learning (ASTRA)}, a theoretically-grounded framework for dynamic satellite topology reconfiguration that builds on an \textbf{online learning formulation} and makes it computationally practical. ASTRA combines an ADMM-based offline solver with efficient online updates for both online gradient descent and online conditional gradient, yielding markedly cheaper constrained updates than generic optimization pipelines. On the theory side, we show that for a relevant class of entry-wise nonzero utility matrices, the objective is strongly convex, which yields logarithmic static regret for online gradient descent, and we further instantiate known dynamic-regret guarantees under inexact ADMM inner loops. Empirically, ASTRA matches or improves topology quality, presenting a good trade-off with computational time on synthetic constellations, and it remains effective on real Starlink data under partial deployment and non-uniform spacing, where idealized structural assumptions break down. These results position ASTRA as an efficient and theoretically grounded approach to topology reconfiguration in realistic Low Earth Orbit networks.

\end{abstract}

\section{Introduction \label{sec:intro}}

Satellite technology has advanced significantly, with modern satellites being used for a plethora of tasks such as weather forecasting~\cite{hertzfeld2004weather} and internet provision~\cite{shaengchart2023starlink, de2015satellite, chen2024concept}. Typically, satellites are organized in constellations and are divided into orbital planes, containing a variable number of satellites, depending on the task at hand. E.g., satellites in geostationary orbit allow for continuous observation of the same geographic region, while satellites in Low Earth Orbit (LEO) are suitable for real-time communication due to low propagation delays, high bandwidth, and higher throughput when compared with other orbital altitudes~\cite{papapetrou2007distributed, han2021dynamic}. Traditionally, for satellites to communicate with each other, they would follow a bent-pipe architecture. A satellite would send a packet to a ground station, which would then send the packet back to the destination satellite when possible. Since deploying a large number of ground stations is not possible due to both geographical and economic constraints, satellites could only communicate with each other when a ground station was visible. 

More recently, the scientific community has focused on the use of inter-satellite links (ISLs) to achieve near-persistent, direct connections with the ground stations. Establishing an ISL involves using the pointing acquisition and tracking (PAT) system. Overall, the use of ISLs allows satellites to communicate among themselves without solely relying on ground stations. Thus, packets can be forwarded directly via ISLs, enhancing the quality of the provided services~\cite{wu2025enhancing}. The maintenance of satellite constellations is crucial to preserve these services. Their growth~\cite{9461407, kulu2024satellite}, with the limited lifespan of the satellites \cite{gonzalo2014challenge}, poses a challenge regarding the need to reconfigure their network every time a new satellite is added or removed. Furthermore, the increase in the number of orbiting objects and, consequently, the risk of collision between them, can potentially impact the consistency and quality of the services provided \cite{guo2025resilience}, leading to both social and economic consequences. It is important that the communication between satellites and the network's topology dynamically adapt to these changes, allowing for stable and more robust services. Nevertheless, a purely dynamic strategy, where satellites would acquire information about the environment and their neighbors every time a new packet needs to be sent, would not be desirable due to the impact it would have on performance. Furthermore, some approaches assume full network symmetry with both satellites and orbital planes uniformly distributed among each other and assume full constellation deployment. Thus, there is a need for an efficient topology configuration algorithm capable of coping with the dynamic behavior of satellite constellations, with minimal assumptions.

\paragraph*{Contributions.} 
Recent work~\cite{norberto2026online} formulated dynamic constellation topology reconfiguration as a convex optimization problem with flexible utility design and showed that generic online methods such as OGD and OCG can be applied to adapt to time-varying settings. Our contribution is not a new problem formulation. Instead, we make that formulation substantially more practical and theoretically sharper through a problem-specific optimization treatment and stronger guarantees. Specifically, we:
\begin{itemize}[noitemsep,topsep=0pt,parsep=0pt,partopsep=0pt,wide=2pt]
\item \textbf{Derive a problem-structured ADMM solver for topology optimization of satellite networks.} We show that the topology optimization problem of Norberto et al.~\cite{norberto2026online} can be solved offline with an ADMM-based algorithm tailored to the Laplacian constraints, replacing generic solver calls by problem-specific updates. Thus, ASTRA presents a problem-structured ADMM solver that includes the closed-form Laplacian projections, retains the outer updates of~\cite{norberto2026online}, and replaces their constrained optimization steps with more efficient ASTRA subroutines.
 \item \textbf{Accelerate the online methods.} Building on the observation of Norberto et al. that OGD and OCG are applicable to this setting, we derive ADMM-based subroutines for their constrained updates, yielding substantially cheaper online iterations than generic optimization pipelines.
 \item \textbf{Establish stronger regret guarantees for this topology objective.} We prove that, for a useful class of entry-wise nonzero utility matrices, the objective is strongly convex. This implies logarithmic static regret for OGD against any fixed topology comparator, strengthening the theoretical understanding of online topology reconfiguration in this setting.
 \item \textbf{Connect the theory to the implemented inexact solver.} We further analyze dynamic regret under inexact ADMM inner loops via feasible-shadow arguments and relate the regret degradation to infeasibility and ADMM primal residuals, bounding dynamic regret when the inner ADMM loop terminates at finite precision.
  \item \textbf{Validate beyond idealized constellation assumptions.} In contrast to classical topology-design methods that rely on full deployment or uniform orbital spacing, we extend the analysis of~\cite{norberto2026online} to non-idealized constellation conditions and dynamic topology changes and show that the resulting framework remains effective on realistic Starlink data under partial deployment and non-uniform spacing. On real Starlink data, ASTRA preserves the connectivity quality of SOTA and attains the largest number of fully connected topologies while offering substantially lower computational cost.
\end{itemize}

\section{Related Work and Background \label{sec:related_work}}

Topology configuration methods can be used to define a topology at a certain instant, given the position of all satellites. Two main approaches emerge from the wide range of methods developed~\cite{xiaogang2016survey}: static approaches, which consider the topology fixed for a certain time interval, and dynamic approaches, which account for possible changes that can occur to satellites and their respective inter-satellite links (ISLs).

Static methods are divided into Virtual Topology (VT)~\cite{sun2020improved, lu2013virtual, 10436098}, and Virtual Node (VN)~\cite{ekici2002distributed, 10404740}. VT approaches divide the system period of the network into small enough \(n\) time slices, where the state of the satellite network is assumed stationary. Subsequently, a substantial number of time slices can be generated. Thus, these approaches deal with the trade-off between storage capacity and performance, as satellites have limited storage. VN approaches assume a network of logical satellite locations, each attributed to the closest real satellite. With the network state fixed, strategies to establish ISLs are used. Currently, +Grid is the most commonly used approach due to its simplicity and runtime. +Grid connects each satellite with its four closest neighbors, where two are from the same orbital plane and two from adjacent orbital planes, assuming that satellites and orbital planes are uniformly separated. Recent works have tried to improve on this method, such as \(\times\)Grid~\cite{mclaughlin2023grid} and Motif~\cite{Network-topology-design-at-27k-km/hour}. \(\times\)Grid~\cite{mclaughlin2023grid} proposes a location-oriented design, where ISLs are established to allow for handovers to occur between connected satellites, thus simplifying them. Motif~\cite{Network-topology-design-at-27k-km/hour} exploits repetitive patterns in satellite networks, reducing traffic and latency. This method can be seen as a generalization of +Grid. Static methods simplify the topology of satellite networks but do not account for the inherent dynamic behavior of satellite networks. Thus, unexpected changes in the topology can lead to link failures and congestion.

Alternatively, dynamic approaches~\cite{papapetrou2007distributed, leyva2021inter, ron2025time, 8688478, pi2022dynamic} allow satellites to acquire information about neighbors and the state of the network to improve the reliability of the established ISLs. Papapetrou et al.~\cite{papapetrou2007distributed} proposed an on-demand protocol to route packets through valid ISLs. It relies on message flooding to discover valid routes, thus raising concerns regarding the PAT time and energy consumption. Leyva-Mayorga et al.~\cite{leyva2021inter} proposed a many-to-many maximum weighted matching problem, where the set of feasible pairs of satellites that maximizes the sum of signal-to-noise ratio is selected. Ron et al.~\cite{ron2025time} formulate the problem of ISL establishment as an optimization problem, where network capacity is maximized while minimizing link churn and latency. To solve this optimization problem, the DoTD algorithm is proposed. Li et al.~\cite{8688478} propose an algorithm that allows satellites to collect information about alive neighbors, ensuring that packets are routed through possible ISLs. Pi et al.~\cite{pi2022dynamic} proposed MA-IILP, an inter-plane ISL planning method based on Multi-Agent Deep Reinforcement Learning, with the objective of maximizing the throughput of the network and minimizing the inter-plane ISL switching rate. Multi-Agent Deep Deterministic Policy Gradient is used to train the agent's networks. Esswein et al.~\cite{EssweinLAMP} propose a network topology that focuses on minimizing the mean propagation delay between satellites by using a backbone network of static links and selecting the set of temporary links that minimizes the mean propagation delay. As the PAT time is in the order of seconds~\cite{wu2025enhancing, bhattacharjee2023laser, ali2026leo}, and satellites have limited battery and processing, a purely dynamic strategy, where satellites would acquire information about the environment and their neighbors every time a new packet needs to be sent, would not be desirable. Hence, there needs to be a trade-off between the performance of these methods and adaptation to the dynamic behavior of the network's topology. Recently, similar to~\cite{ron2025time}, Norberto et al.~\cite{norberto2026online} formulated the network topology configuration problem as a convex optimization problem, where the topology of the network is modeled as a Laplacian matrix \(X\), guided by a utility matrix that can be arbitrarily defined to meet any desired properties and/or requirements of the operators (more details in Section~\ref{sec:opt-problem-formulation}). The proposed general formulation allows for broad scope and flexibility for different types of missions. In addition, the authors show that online algorithms can be easily applied to the derived optimization problem, allowing for automatic adaptation. Although not named by the authors, for convenience we will denote this method by GUTO (General Utility Topology Optimization). In this paper, we build upon their work and present an efficient implementation, allowing us to deal with bigger constellations and derive theoretical guarantees, thus bridging the gap between theory and application.

%\section{Background}
%\label{sec:background}

\paragraph*{Online Convex Optimization. \label{sec:oco_introd}}
We model dynamic topology reconfiguration within the framework of online convex optimization (OCO)~\cite{cesa2006prediction,hazan2016introduction,orabona2019modern}. This setting has been vastly studied~\citep {zinkevich2003online,bubeck2012regret,hazan2012projection,zhang2018adaptive,hazan2020faster,zhao2020dynamic,hazan2021boosting,gatmiry2023projection}, demonstrating applicability in a vast number of different problems, such as portfolio selection~\citep{cover1991portfolio,hazan2007logarithmic}, stochastic optimization~\citep{duchi2011adaptive,hazan2011stochastic}, and control~\citep{cassel2022control,gradu2023control,nonhoff2026control}. At each round $t$, the learner selects a feasible decision $x_t \in \mathcal{K}$, after which a convex loss $f_t$ is revealed. Performance is measured by the regret, $\mathrm{Regret}_T = \sum_{t=1}^T f_t(x_t) - \sum_{t=1}^T f_t(u_t)$, where $u_t$ denotes a comparator decision sequence. When the comparator is the best fixed decision in hindsight, this yields \emph{static regret}; when the comparator may vary over time, this yields \emph{dynamic regret}. Two standard algorithmic families are especially relevant here: projection-based methods such as Online Gradient Descent (OGD) and projection-free methods such as Online Conditional Gradient (OCG)~\cite{hazan2016introduction}. Norberto et al.~\cite{norberto2026online} showed that both can be applied to topology reconfiguration under the GUTO formulation. In our setting, Theorem~\ref{thm:strong-convexity-proof} shows that, for a useful class of utility matrices, the losses are strongly convex; therefore, by the standard OGD guarantee for strongly convex losses~\cite[Theorem~3.3]{hazan2016introduction}, ASTRA-OGD attains logarithmic static regret. We defer a brief tutorial on the OCO framework, including the standard distinction between static and dynamic regret, to Appendix~\ref{sec:oco_appendix}.

\paragraph*{Alternating Direction Method of Multipliers (ADMM). \label{sec:admm_introduction}}
We use ADMM as the optimization backbone of ASTRA, both for the offline solver and for the constrained subroutines used by the online methods. Since our contribution is not a new variant of ADMM itself, we defer a brief tutorial on the generic ADMM framework to Appendix~\ref{sec:ADMM} and focus in the main text on the problem-specific splitting induced by the Laplacian and box constraints.

\section{ASTRA: Adaptive Satellite Topology via Regret-Aware learning}
\label{sec:astra-section}

In this section, we present the main features of our novel framework ASTRA. We start by summarizing GUTO, the optimization problem defined in~\cite{norberto2026online}, and then show that for a specific class of utility matrices, the optimization problem becomes strongly convex. Posteriorly, we present our ADMM-based efficient implementation to solve GUTO and also show that ADMM can be applied to the online learning algorithms (OGD and OCG), allowing for even faster solutions.

\subsection{General Utility Topology Optimization (GUTO)}
\label{sec:opt-problem-formulation}

The optimization problem proposed in~\cite{norberto2026online} can be compacted into the form
\begin{equation}
    \label{eq:net-top-optimization-problem}
    \begin{aligned}
        \underset{X \in \mathbb{R}^{n \times n}}{\text{minimize}} \quad 
        & \frac{1}{2}\|(X-P)\odot U\|^2_F - \textbf{Tr}(\Lambda\odot X) \\
        \text{subject to} \quad 
        & X = X^T , \quad X \mathbf{1} = \mathbf{0}, \quad B^{\text{lower}}\le X \le B^{\text{upper}},
    \end{aligned}
\end{equation}
where \(X\) is the Laplacian matrix representing the topology of the network, $\Lambda$ is a diagonal matrix used to encourage connections between satellites, \(\|\cdot\|_F\) represents the Frobenius norm, \(\textbf{Tr}(\cdot)\) the trace operator, \(\odot\) the entry-wise product, and $\mathbf{0}$ and $\mathbf{1}$ denote the all-zeros and all-ones vector, respectively. The matrix $B^{\text{upper}} = \text{diag}(d_1^{\max}, \dots, d_n^{\max})$ is a diagonal matrix, where $d_i^{\max}$ denotes the maximum degree of each node/satellite, and $B^{\text{lower}}= I -\mathbf{1}\mathbf{1}^T$, where $I$ represents the identity matrix. These two matrices further constrained the structure of the Laplacian matrix, guaranteeing that the off-diagonal values are between $0$ and $-1$, and the diagonal values are between zero and the allowed maximum degree of each node (i.e., the maximum number of connections the satellite can establish). The connectivity matrix \(P \in \mathbb{R}^{n \times n}\) represents the set of feasible connections a satellite can have. If the connection between satellite \(i\) and satellite \(j\) is feasible, then \(P_{ij}=-1\); otherwise, \(P_{ij}=0\). In this work, to define which connections are feasible, we employ the two approaches to approximate the field of view (FOV) of a satellite introduced in~\cite{norberto2026online} and limit the allowed maximum distance for a connection to be established by $8,000$ km, as ISL properties are directly affected by distance~\cite{ali2026leo}. The utility matrix \(U \in \mathbb{R}^{n \times n}\) individually assigns different weights to each possible ISL. These utilities can be constructed based on different metrics, such as packet drop probability~\cite{taleb2009}, network capacity~\cite{ron2025time}, energy consumption on packet forwarding~\cite{10436098}, or structural preferences, such as connections between satellites in the same orbital plane. Depending on the network’s intended task, we can combine different utilities $U_i$ through a convex combination, i.e., for $m$ different utilities, we can define $U = \sum_{i=1}^{m} w_i U_i$ such that $w_i \geq 0$ and $\sum_{i=1}^{m} w_i = 1$.
To create the underlying graph from the resulting Laplacian matrix, we first filter the entries of \(X\) with a threshold \(\tau\). This threshold can be seen as how conservative the operator is when considering which ISLs are proposed by the method. Afterwards, we select the best possible ISLs, until each node/satellite achieves its maximum degree or until there are no more recommended connections. The effect of this discretization is further explored in Appendix~\ref{app:mip_regret_comp}.

\paragraph*{Logarithmic regret for entry-wise non-zero utilities.} As our first contribution, we show that GUTO, summarized in~\eqref{eq:net-top-optimization-problem}, is strongly-convex for the set of utility matrices that satisfy $U_{ij} \not= 0$ for all $i,j \in \{1, 2, \dots, n\}$. This result is summarized in Theorem~\ref{thm:strong-convexity-proof}, and the proof and auxiliary lemmas are presented in Appendix~\ref{sec:proof_str_convexity}.
\begin{theorem}
    \label{thm:strong-convexity-proof}
    Consider the space of real square matrices $\mathbb{R}^{n \times n}$ associated with the Frobenius inner product $\langle X, Y \rangle_F = \sum_{i=1}^n \sum_{j=1}^n X_{ij} Y_{ij}$. Let $f : \mathbb{R}^{n \times n} \to \mathbb{R}$ be defined as $f(X) = \frac{1}{2}\|(X-P)\odot U\|^2_F - \textbf{Tr}(\Lambda \odot X)$, following the form of the objective function in~\eqref{eq:net-top-optimization-problem}. Then, if $U_{ij} \not= 0$ for all $i,j \in \{1, 2, \dots, n\}$, $f$ is $\mu$-strongly convex, with $\mu = \min\{U_{ij}^2\}$.
\end{theorem}
A simple example that satisfies this condition is a utility matrix where each entry $(i,j)$ expresses the reciprocal of the distance between nodes $i$ and $j$. Thus, we can apply Theorem~\ref{thm:strong-convexity-ogd-regret}, showing that OGD obtains static logarithmic regret when applied to cost functions involving this type of utility matrices. In addition to the strong theoretical guarantees, we show, in Section~\ref{sec:results}, that this simple utility matrix achieves good results in both synthetic and real-world scenarios.

\begin{remark}
    Entry-wise non-zero utility matrices are a sufficient condition for strong convexity and not an assumption for ASTRA to operate. For utilities outside this class, the objective remains convex, leading online algorithms to retain their standard regret guarantees. E.g., the standard \(O(\sqrt{T})\) static regret bound for OGD. Moreover, an entry-wise nonzero utility matrix does not imply a dense physical topology. Sparsity of the selected links is still enforced by the time-varying field-of-view matrix and maximum-degree constraints. Entries expressing weak preferences may also be arbitrarily small. Therefore, the strong convexity constant, and hence the associated regret constant, deteriorates as they approach zero.
\end{remark}

\subsection{ADMM-based efficient implementation}
\label{sec:admm-based-efficient-implementation}

We can express Problem~\eqref{eq:net-top-optimization-problem} in ADMM form as
\begin{equation}
    \label{eq:general-admm-problem-formulation}
    \begin{aligned}
        \underset{X, Z}{\text{minimize}} \quad f(X) + g(Z) \quad \text{subject to} \quad X - Z = \mathbf{0_{n \times n}} ,
    \end{aligned}
\end{equation}
where $\mathbf{0_{n \times n}}$ denotes the all-zeros $n \times n$ matrix, $g(Z) = i_{B}(Z)$ is the indicator function of set \(B \coloneqq \{X \in \mathbb{R}^{n \times n} : B^{\text{lower}} \leq X \leq B^{\text{upper}} \}\), and $f(X) = \frac{1}{2} \|(X-P) \odot U\|_F^2 - \mathbf{Tr}(\Lambda \odot X)$, where $\text{dom} f = \{X \in \mathbb{R}^{n \times n} : X = X^T, X \mathbf{1} = \mathbf{0} \}$. Here, $f$ is the original objective with a restricted domain. The augmented Lagrangian (in the scaled form~\citep{boyd2011admm}) is given by $L_{\rho}(X,Z,V) = f(X) + g(Z) + \frac{\rho}{2} \left\| X - Z + V \right\|_F^2$, where $V$ is the scaled dual variable. Then, the general formulation of ADMM for this problem is given by
\begin{equation}
    \label{eq:general-admm}
    \begin{aligned}
        X^{k+1} &= \underset{X \in \text{dom} f}{\arg\min} \left\{ f(X) + \frac{\rho}{2} \left\| X - Z^k + V^k \right\|_F^2 \right\} \\
        Z^{k+1} &= \underset{Z \in \mathbb{R}^{n \times n}}{\arg\min} \left\{ g(Z) + \frac{\rho}{2} \left\| X^{k+1} - Z + V^k \right\|_F^2 \right\} \\
        V^{k+1} &= V^k + \left( X^{k+1} - Z^{k+1} \right)
    \end{aligned}
\end{equation}

\input{Files/net-top-config-off}
\input{Files/net-top-config-on}

\section{Numerical Experiments }
\label{sec:results}
Throughout this section, we compare our method with the state-of-the-art on both synthetic and real-world data and present metrics for the graphs generated, viz, total number of edges in the graph, ``E''; average degree of a satellite, ``A-Deg''; number of connected components, ``CC''; average shortest path, ``A-SP''; and the wall clock time in seconds, ``Time (s)''. Since graph-level metrics measure topology quality but do not replace routing-based metrics, such as packet-level throughput and latency, we provide the degree distribution for both scenarios. Wall clock time experiments were executed on an M5 Pro MacBook Pro CPU, with Starlink metric experiments executed in 4 cores of an Intel(R) Xeon(R) Gold 6330. Our synthetic dataset consists of an Iridium-like constellation with 6 elliptical orbital planes, each one with 11 satellites per plane. For real-world data, we consider a constellation with 1591 Starlink satellites, distributed across 50 orbital planes, containing between 14 and 64 satellites, whose state was retrieved every five minutes during a period of 24 hours~\footnote{Code is available at \hyperref[]{https://github.com/JDNorberto/ASTRA-NeurIPS26}}.

As utility matrices, we use the reciprocals of the distance between satellites and a structural utility that encourages connections with the three closest orbital planes. Formally \(U^{\text{sat\_dist\_rec}}_{ij} = \frac{1}{d_{ij}}\), where $d_{ij}$ represents the distance between satellite \(i\) and \(j\). The diagonal values, where $d_{ij} = 0$, are bounded by \(1\times 10^{-5}\). The utility \(U^{\text{3-closest}}\) pairs satellites and orbital planes in a round-robin fashion, where the only considered orbital planes are the three closest ones. Formally, \(U^{\text{3-closest}}_{ij} = \frac{1}{d_{ij}}\) if satellite \(j\) belongs to the orbital plane considered by satellite \(i\); otherwise, \(U^{\text{3-closest}}_{ij} = 0\). We combine these matrices through a convex combination, i.e., \(U = \lambda U^{\text{3-closest}} + (1-\lambda) U^{\text{sat\_dist\_rec}}\), with \(\lambda \in [0, 1)\)~\footnote{To ensure strong-convexity by Theorem~\ref{thm:strong-convexity-proof}, $\lambda$ cannot be equal to $1$, to ensure $U$ is an entry-wise non-zero matrix.}. Each row of the utility matrices is normalized by its maximum value. For the ADMM sub-routines we consider the stopping criteria $\frac{\left|\left| V^{k}-V^{k-1}\right|\right|}{\left|\left|V^{k-1}\right|\right|} < \epsilon = 10^{-7}$, where $V^k$ is the dual variable.

\paragraph*{Ablation study.} Our approach, ASTRA, relies on different parameters, such as \(\lambda\), \(\rho\), \(\text{proj}_{its}\), and FOV approximation (in Appendix~\ref{sec:astra-parameters} we explain the meaning of each parameter). In Appendix~\ref{app:ablation_studies}, we performed ablations in order to select the best parameters. In the following experiments, ASTRA uses \(\lambda = 0.6\), \(\rho=1\), \(\text{proj}_{its}=5\), and ``H+C'' as the approximation of the FOV of a satellite. For GUTO, we use the same \(\lambda\) and approximation. For both, we set \(\Lambda = 10^{-4} I\).

\paragraph*{State-of-the-art comparison.} In Table~\ref{tab:iridium-like_constellation}, we summarize the results of ASTRA and the different methods we compare it with in the Iridium-like constellation. We see that +Grid and xGrid are the fastest methods, although the latter presents the lowest connectivity metrics and the highest shortest path. Both Motif~\cite{Network-topology-design-at-27k-km/hour} and DoTD~\cite{ron2025time} achieve the best network connectivity out of all the methods, although Motif takes a much longer time to achieve this. Both GUTO and ASTRA achieve smaller average shortest path length and comparable connectivity with DoTD and Motif. ASTRA attains no degree-0 or degree-1 satellites and only \(0.5_{\pm 0.5}\) degree-2 satellites; almost all nodes have a degree of 3 or 4 (Tab.~\ref{tab:avg_degree_distribution}). This provides additional evidence against structural bottlenecks and shows that under idealized conditions, ASTRA is capable of obtaining the same connectivity characteristics as GUTO, taking advantage of the available ISLs. Furthermore, GUTO and ASTRA are equivalent, with ASTRA requiring approximately 22\% less execution time. Under comparable stopping criteria, the performance of ASTRA with respect to GUTO increases even further, up to three times faster (Appendix~\ref{app:time-eps-equal}). Comparison with learning-based methods is presented in Appendix~\ref{app:learning-based_comp}.

\input{Tables/deg_dist_iridium}

\input{Tables/iridium_ours_sota}

\paragraph*{OGD with fine-tuned step-size for optimal dynamic regret.} As described in Section~\ref{sec:online-derivation}, we can retrieve the path length of a dynamic comparator to deploy a fine-tuned OGD. Exploiting the previous experiment, we consider the offline ASTRA as the dynamic comparator, for which we obtain $P(U_1, \dots, U_T) \approx 9393$. In Fig.~\ref{fig:regret_comparison}, we show that a deployed OGD with the fine-tuned step-size $\eta = \sqrt{\frac{P(U_1, \dots, U_T)}{T}}$ achieves a smaller regret, therefore superior results, than a deployed OGD with the standard step-size $\eta = \frac{1}{\sqrt{T}}$ used in OCO.

\input{Figures_tex/regret_comparison}

\paragraph*{Changes in constellation topology}
One of the main drawbacks of state-of-the-art methods is that, to cope with changes in the satellite network, such as satellite removals or additions, a topology needs to be computed from scratch, which is computationally expensive and time-consuming. We show that online algorithms with ASTRA subroutines can adapt to these changes without computing a topology from scratch; we consider the same Iridium-like constellation and randomly select six satellites to be removed. For ASTRA-OCG, we consider a step-size \(\eta = \frac{1}{T^{\frac{3}{4}}}\), and $\sigma_t = \min \left\{1, \frac{2}{t^{\frac{1}{2}}} \right\}$~\cite{hazan2016introduction}. For ASTRA-OGD, we consider the dynamic regret step size $\eta = \sqrt{\frac{P(U_1, \dots, U_T)}{T}} = \sqrt{\frac{9393}{T}}$. 

\input{Figures_tex/online_adapt_rem_add}

Fig.~\ref{fig:adapt_online_rem_add} shows the smoothed average residual per entry, defined as
\begin{align*}
    \text{Avg-RpE} = \frac{\left\| X_{\text{on}} - X_{\text{GUTO}}\right\|_F^2}
{\#\left[(x_{ij} \neq 0), \forall_{x_{i, j}} \in X_{\text{GUTO}} \right]},
\end{align*}
where \(X_{\text{on}}\) is the Laplacian matrix proposed by the online algorithms and \(X_{\text{GUTO}}\) denotes the matrix proposed by GUTO. We see that ASTRA-OGD easily adapts to the addition and removal of satellites, while ASTRA-OCG seems to struggle with satellite addition, as the residual slightly increases, and different topologies are proposed for the same set of satellites. This discrepancy in performance is expected, as OCG is tailored to control, at each iteration, the cumulative sum of past costs.

\paragraph*{Starlink data.}

In this section we evaluate our method and compare it to the state-of-the-art on Starlink data. At each time step, the algorithms are applied to generate a candidate topology. The metrics reported in each table correspond to the mean values and their respective standard deviations across all time steps. Since the average shortest path is defined only for topologies consisting of a single connected component, we report this metric exclusively for the topologies that satisfy this condition, thus denoting it as ``A-SP*''. In this scenario, we do not evaluate Motif, as its implementation is not designed to support constellations with orbital planes containing a varying number of satellites. We use the best-performing parameters on the Iridium-like constellation to evaluate ASTRA and GUTO on Starlink data, with the exception of \(\text{proj}_{its}= 20\) and \(\tau = -0.45\), due to increased complexity.
A key qualitative distinction in the Starlink setting is connectivity. While several baselines are faster or competitive on some local graph statistics, only GUTO and ASTRA attain fully connected topologies in at least some runs (Tab.~\ref{tab:offline_starlink}). This is important because the average shortest path is only meaningful as a full-network connectivity metric in that regime. Accordingly, the most relevant comparison is between GUTO and ASTRA: ASTRA preserves the connectivity quality of GUTO while reducing computational cost. By contrast, the remaining baselines often yield fragmented topologies, which limits direct comparison in terms of end-to-end path efficiency. Regarding the degree distribution in this setting, Tab.~\ref{tab:avg_degree_distribution_strl} shows that in a more unstructured case, the existence of bottlenecks is more probable, as the number of satellites that have degree-0 and degree-1 is non-zero. Nevertheless, we see that the majority of the satellites, around 94\% and 89\% when using the ``H'' and ``H+C'' approximations, respectively, have an average degree greater than or equal to 3. GUTO and ASTRA are capable of using the majority of the available ISLs in both idealized and non-idealized constellation conditions, with ASTRA closely matching GUTO's distribution.

\input{Tables/offline_starlink} 

\input{Tables/deg_dist_starlink}

In terms of wall-clock time (Tab.~\ref{tab:times_offline}), ASTRA achieves 83\% faster execution time on average when compared with DoTD and 15\% when compared with GUTO, obtaining a similar loss function value to GUTO, up to a precision of \(10^{-3}\). Hence, ASTRA should be viewed as an efficient solver that retains the desirable connectivity behavior of GUTO under realistic partial-deployment and non-uniform-spacing conditions. Both +Grid and \(\times\)Grid are significantly faster than other methods due to the trade-off between computational cost and how dynamic an approach is. Further experiments on ASTRA-OGD and ASTRA-OCG are presented in App.~\ref{app:on_alg_on_strl}, cementing ASTRA as an efficient and theoretically grounded framework for dynamic satellite topology reconfiguration.
\input{Tables/times_offline}
\paragraph*{Limitations \label{sec:limitations.}}
ASTRA focuses on network topology configuration and does not address routing in these networks. Consequently, data transmission aspects and routing efficiency remain unexplored. Base stations were not considered in this study. Therefore, results may vary depending on their number and location. Although online algorithms can be deployed on each satellite, ensuring consistency across all satellites in a constellation is non-trivial and may introduce additional challenges. From a computational perspective, the reliance on matrix-based representations incurs quadratic spatial complexity, as satellites must store and process these matrices.

\section*{Conclusions \label{sec:conclusion}}

We introduced ASTRA, a framework for dynamic satellite topology reconfiguration that combines an ADMM-based offline solver with efficient ADMM-based constrained updates for online methods. We showed that, for a useful class of utility matrices, the topology objective is strongly convex, yielding logarithmic static regret for OGD, and we related dynamic regret under inexact inner loops to infeasibility and ADMM residuals. In experiments, ASTRA matches SOTA, presenting a good trade-off between computational time and topology characteristics, adapts to satellite additions and removals without recomputing the topology from scratch, and remains effective on realistic Starlink data under partial deployment and non-uniform spacing, where only GUTO and ASTRA attain fully connected topologies in some runs. Overall, the results support online learning as a practical route towards realistic topology reconfiguration in large LEO networks.

\paragraph{Broader impacts.}
This work may have a positive societal impact by improving the robustness of satellite constellations that support communication and Earth observation, helping maintain continuity under failures and changing network conditions. Potential negative impacts include dual-use in military or surveillance settings and degraded connectivity if recommendations are deployed without sufficient validation.

\paragraph{Acknowledgments}
This work is supported by NOVA LINCS (UID/04516/2025) with the financial support of FCT.IP (DOI:10.54499/UID/04516/2025). This research was carried out under Project “Artificial Intelligence Fights Space Debris” Nº C626449889-0046305 co-funded by Recovery and Resilience Plan and NextGeneration EU Funds, www.recuperarportugal.gov.pt.

\input{Figures_tex/bar_logos}

\bibliographystyle{abbrv}
\bibliography{mybib}

\clearpage
\appendix

\section{Online Convex Optimization Background}
\label{sec:oco_appendix}
Online Convex Optimization is an important framework for modeling sequential decision-making~\cite{hazan2016introduction, orabona2019modern}. This framework can be seen as a structured repeated game, where a player iteratively chooses an action over some convex set \(\mathcal{K}\), while an adversary selects a convex cost function \(f_t\) \textit{a posteriori}, causing the player to incur a cost associated with that decision. The objective of the player is to perform as well as a comparator strategy over all rounds of interaction.

The main performance metric is the regret, defined as the difference between the cumulative cost incurred by the player and that incurred by the comparator, namely
$$\mathrm{Regret}_T
=
\sum_{t=1}^T f_t(x_t)
-
\sum_{t=1}^T f_t(u_t),$$
where $x_t$ denotes the decision selected by the player at time step $t$, and $u_t$ denotes the comparator decision at the same round. When the comparator selects the best fixed decision in hindsight, i.e.,
$$u_t
=
\arg\min_{x \in \mathcal{K}} \sum_{\tau=1}^T f_\tau(x)
\qquad \text{for all } t,$$
the metric is called \emph{dynamic regret}, and is further denoted by \(\mathrm{D\text{-}Regret}_T\).

Two important classes of algorithms are projection-based and projection-free algorithms~\cite{hazan2016introduction}. The former ensure feasibility by projecting the iterate onto the decision set \(\mathcal{K}\). However, in some cases this projection can be computationally expensive. Projection-free algorithms can then serve as an efficient alternative, avoiding that projection by solving a linear optimization problem instead. The simplest projection-based algorithm is Online Gradient Descent (OGD), where, at each time step \(t\), one performs a gradient step with respect to the revealed cost function \(f_t\) and then projects onto \(\mathcal{K}\). A simple example of a projection-free algorithm is the Online Conditional Gradient (OCG) algorithm, which, at each iteration \(t\), solves a linear optimization problem based on the aggregate sum of all previous cost functions together with a strongly convex and smooth regularization term.

For general convex cost functions $f_t$  it is known that OGD and OCG attain regret of $O(\sqrt{T})$ and $O(T^{3/4})$, respectively, when compared to any fixed comparator~\cite{hazan2016introduction}. These two methods were employed by Norberto et al.~\cite{norberto2026online} for the problem of network topology configuration. In our setting, for a useful class of utility matrices, Theorem~\ref{thm:strong-convexity-proof} shows that the resulting losses are strongly convex. Hence, by the standard OGD guarantee for strongly convex losses, OGD attains logarithmic static regret. For completeness, we restate below the classical result that serves as the basis for our analysis.
\begin{theorem}[{\cite[Theorem 3.3]{hazan2016introduction}}]
    \label{thm:strong-convexity-ogd-regret}
    For \(L\)-Lipschitz continuous, \(\alpha\)-strongly convex cost functions \(f_t\), Online Gradient Descent with step sizes \(\eta_t = \frac{1}{\alpha t}\) attains, for all \(T \geq 1\), the static regret bound
    \[
    \mathrm{Regret}_T \leq \frac{L^2}{2\alpha}(1+\log T).
    \]
\end{theorem}

\section{ADMM Background}
\label{sec:ADMM}
The Alternating Direction Method of Multipliers (ADMM) is a classical splitting method for optimization problems of the form
\begin{equation}
\label{eq:admm_generic}
\begin{aligned}
\min_{x,z} \quad & f(x)+g(z)\\
\text{s.t.} \quad & Ax+Bz=c,
\end{aligned}
\end{equation}
where $x\in\mathbb{R}^n$, $z\in\mathbb{R}^m$, $A\in\mathbb{R}^{p\times n}$, $B\in\mathbb{R}^{p\times m}$, and $c\in\mathbb{R}^p$. ADMM combines variable splitting with dual updates, and is especially useful when the objective or constraints decompose into parts that are easier to handle separately.

The augmented Lagrangian is
\[
L_\rho(x,z,y)
=
f(x)+g(z)+y^\top(Ax+Bz-c)+\frac{\rho}{2}\|Ax+Bz-c\|^2,
\]
where $y$ is the dual variable and $\rho>0$ is a penalty parameter. ADMM alternates minimization with respect to $x$ and $z$, followed by a dual ascent step:
\begin{align}
x^{k+1} &= \arg\min_x L_\rho(x,z^k,y^k),\\
z^{k+1} &= \arg\min_z L_\rho(x^{k+1},z,y^k),\\
y^{k+1} &= y^k + \rho\,(Ax^{k+1}+Bz^{k+1}-c).
\end{align}

In ASTRA, we use a consensus form of ADMM in which the structured Laplacian constraints are handled in one block and the box constraints in the other; the resulting specialization is given in Section~\ref{sec:astra-section}.

\section{Proof of strong convexity}
\label{sec:proof_str_convexity}

\begin{lemma}
    \label{lemma:strong-convexity-equivalence}
    Let $f : \mathcal{H} \to \mathbb{R}$ be a function defined in the real Hilbert space $\mathcal{H}$, with the norm $\| \cdot \|_{\mathcal{H}} \coloneqq \sqrt{ \langle \cdot, \cdot \rangle_{\mathcal{H}}}$ induced by the inner product $ \langle \cdot, \cdot \rangle_{\mathcal{H}}$. The function $f$ is $\mu$-strongly convex, for some parameter $\mu > 0$, if and only if for every $z \in \mathcal{H}$ and every direction $v \in \mathcal{H}$, the function $g : \mathbb{R} \to \mathbb{R}$, defined as $g(t) = f(z + t v)$, is $\left( \mu \| v \|_{\mathcal{H}}^2 \right)$-strongly convex.
\end{lemma}
\begin{proof}
    First, assume that $f$ is $\mu$-strongly convex. Therefore, for $x, y \in \mathcal{H}$ and $\lambda \in [0,1]$, we have
    \begin{equation}
        \label{eq:strong-convexity-general}
        \begin{aligned}
            f(\lambda x + (1 - \lambda) y) \leq \lambda f(x) + (1 - \lambda) f(y) - \frac{\mu}{2} \lambda (1 - \lambda) \|x - y\|_{\mathcal{H}}^2 .
        \end{aligned}
    \end{equation}
    Let $t_1, t_2 \in \mathbb{R}$. Then, consider $x = z + t_1 v$ and $y = z + t_2 v$, for any $x \in \mathcal{H}$ and direction $v \in \mathcal{H}$. Thus, we have
    \begin{equation}
        \label{eq:functions-equality}
        \begin{aligned}
            f(\lambda x + (1 - \lambda) y) &= f(\lambda (z + t_1 v) + (1 - \lambda) (z + t_2 v)) \\
            &= f(z + \lambda t_1 v + (1 - \lambda) t_2 v) \\
            &= g(\lambda t_1 + (1 - \lambda) t_2)
        \end{aligned}
    \end{equation}
    Thus, applying the strong convexity definition presented in~\eqref{eq:strong-convexity-general}, we have
    \begin{align*}
        g(\lambda t_1 + (1 - \lambda) t_2) &= f(\lambda x + (1 - \lambda) y) \\
        &\leq \lambda f(x) + (1 - \lambda) f(y) - \frac{\mu}{2} \lambda (1 - \lambda) \|x - y\|_{\mathcal{H}}^2 \\
        &= \lambda g(t_1) + (1 - \lambda) g(t_2) - \frac{\mu}{2} \lambda (1 - \lambda) \|(t_1 - t_2) v\|_{\mathcal{H}}^{2} \\
        &= \lambda g(t_1) + (1 - \lambda) g(t_2) - \frac{\mu \| v \|_{\mathcal{H}}^{2}}{2} \lambda (1 - \lambda) (t_1 - t_2)^2
    \end{align*}
    Conversely, assume $g$ is $\left( \mu \| v \|_{\mathcal{H}}^2 \right)$-strongly convex, therefore, for $t_1, t_2 \in \mathbb{R}$ and $\lambda \in (0,1)$, we have
    \begin{align*}
        g(\lambda t_1 + (1 - \lambda) t_2) &\leq \lambda g(t_1) + (1 - \lambda) g(t_2) - \frac{\mu \| v \|_{\mathcal{H}}^{2}}{2} \lambda (1 - \lambda) (t_1 - t_2)^2 \\
        &= \lambda f(z + t_1 v) + (1 - \lambda) f(z + t_2 v) - \frac{\mu}{2} \lambda (1 - \lambda) \|z - z + (t_1 - t_2) v\|_{\mathcal{H}}^{2} \\
        &= \lambda f(z + t_1 v) + (1 - \lambda) f(z + t_2 v) - \frac{\mu}{2} \lambda (1 - \lambda) \|(z + t_1 v) - (z + t_2 v)\|_{\mathcal{H}}^{2} \\
    \end{align*}
    By~\eqref{eq:functions-equality}, we have $f(\lambda (z + t_1 v) + (1 - \lambda) (z + t_2 v)) = g(\lambda t_1 + (1 - \lambda) t_2)$, which completes the proof.
\end{proof}
\begin{remark}
    Note that the case where $v = 0$ is trivial. While the parameter $\mu \| v \|_{\mathcal{H}}^{2}$ is not positive in this case, the strong convexity is preserved in the original function, since it reflects the case where $x = y$.
\end{remark}
\begin{theorem}
    Consider the space of real square matrices $\mathbb{R}^{n \times n}$ associated with the Frobenius inner product $\langle X, Y \rangle_F = \sum_{i=1}^n \sum_{j=1}^n X_{ij} Y_{ij}$. Let $f : \mathbb{R}^{n \times n} \to \mathbb{R}$ be defined as $f(X) = \frac{1}{2}\|(X-P)\odot U\|^2_F - \textbf{Tr}(\Lambda\odot X)$, where $\Lambda$ is a diagonal matrix and $\odot$ denotes the Hadamard product. Then, if $U_{ij} \not= 0$ for all $i,j \in \{1, 2, \dots, n\}$, $f$ is $\mu$-strongly convex, with $\mu = \min\{U_{ij}^2\}$.
\end{theorem}
\begin{proof}
    By Lemma~\ref{lemma:strong-convexity-equivalence}, we have that the function $f$ is $\mu$-strongly convex, for some parameter $\mu > 0$, if and only if $g(t) = f(X + t V)$ is $\left( \mu \|V\|_F^{2} \right)$-strongly convex, for every $X \in \mathbb{R}^{n \times n}$ and every direction $V \in \mathbb{R}^{n \times n}$. To prove the strong convexity, we will resort to the second-order condition, which states that function $g$ is $\left( \mu \|V\|_F^{2} \right)$-strongly convex if and only if $g''(t) \geq \mu \|V\|_F^{2} > 0$ (see~\citep{boyd2004convex}). First, let us see that
    \begin{align*}
        g(t) &= \frac{1}{2}\|(X + tV -P)\odot U\|^2_F - \textbf{Tr}(\Lambda \odot (X + t V)) \\
        &= \frac{1}{2}\|(X + tV -P)\odot U\|^2_F - \textbf{Tr}(\Lambda \odot X) - t \ \textbf{Tr}(\Lambda \odot V) \\
        &= \frac{1}{2} \|(X - P) \odot U\|_F^2 + t \langle (X - P) \odot U, V \odot U \rangle_F + \frac{1}{2} t^2 \| V \odot U \|_F^2 - \textbf{Tr}(\Lambda \odot X)\\
    & \qquad - t \ \textbf{Tr}(\Lambda \odot V) ,
    \end{align*}

    where the second equality results from the linearity of the trace operator~\citep{axler2024linear}. Taking the first derivative, we have
    \begin{align*}
        g'(t) = \langle (X - P) \odot U, V \odot U  \rangle_F + t \| V \odot U \|_F^2 - \textbf{Tr}(\Lambda \odot V) ,
    \end{align*}
    and, consequently, the second derivative is
    \begin{align*}
        g''(t) &= \| V \odot U \|_F^2 \\
        &= \sum_{i=1}^{n} \sum_{j=1}^{n} U_{ij}^2 V_{ij}^2 \\
        &\geq \min\{U_{ij}^2\} \| V \|_F^2 
    \end{align*}
    Therefore, for $f$ to be strongly convex with parameter $\mu = \min\{U_{ij}^2\} > 0$, we have that $U_{ij} \not= 0$ for all $i,j \in \{1, 2, \dots, n\}$.
\end{proof}

\section{Proof of Lemma 3.2}
\label{sec:lemma_proj}

\begin{proof}
    Consider the optimization problem describing the projection of a matrix Y onto set $K$:
    \begin{align*}
        \underset{X}{\text{minimize}} &\quad \frac{1}{2} \|X - Y\|_F^2 \\
        \text{subject to} &\quad X=X^T \\
        &\quad X \mathbf{1} = \mathbf{0}
    \end{align*}
    The Lagrangian of this function is given by
    \begin{align*}
        L(X, V, \mu) = \frac{1}{2} \|X - Y\|_F^2 + \langle V, X - X^T \rangle + \mu^T X \mathbf{1} .
    \end{align*}
    The stationary condition, i.e., $\nabla_X L(X, V, \mu) = 0$, can be rearranged to obtain
    \begin{equation}
        \label{eq:stationarity-condition-X}
        \begin{aligned}
            X = Y + (V^T - V) - \mu \mathbf{1}^T .
        \end{aligned}
    \end{equation}
    It is a known fact that a square matrix $A$ can be uniquely decomposed into a symmetric part $S(A) = \frac{1}{2}(A + A^T)$ and a skew-symmetric part $C(A) = \frac{1}{2}(A - A^T)$, i.e., $A = S(A) + C(A)$~\citep{horn2012matrix}. At optimality, we know that $X$ is symmetric, then, by considering the symmetric part, applying it to Eq.~\eqref{eq:stationarity-condition-X}, we have
    \begin{equation}
        \label{eq:X-symmetric}
        \begin{aligned}
            X = S(Y) - \frac{1}{2} (\mu \mathbf{1}^T + \mathbf{1} \mu^T) .
        \end{aligned}
    \end{equation}
    Now, enforcing the zero row sum, we have
    \begin{equation}
        \label{eq:mu-first-derivation}
        \begin{aligned}
            X \mathbf{1} = 0 &\iff S(Y) \mathbf{1} - \frac{1}{2} (\mu \mathbf{1}^T \mathbf{1} + \mathbf{1} \mu^T \mathbf{1}) = 0 \\
            &\iff \frac{1}{2} (n \mu + \mathbf{1} \mu^T \mathbf{1}) = S(Y) \mathbf{1} \\
            &\iff \mu = \frac{2}{n} S(Y) \mathbf{1} - \frac{r}{n} \mathbf{1} ,
        \end{aligned}
    \end{equation}
    where $r = \mathbf{1}^T \mu$ (note that $r$ is a scalar and $\mathbf{1}^T \mu = \mu^T \mathbf{1}$). From the result above, substituting $\mu$ in this expression of $r$, we have that
    \begin{align*}
        r &= \frac{2}{n} \mathbf{1}^T S(Y) \mathbf{1} - \frac{r}{n} \mathbf{1}^T \mathbf{1} \iff \\
        r &= \frac{2}{n} \mathbf{1}^T S(Y) \mathbf{1} - \frac{r}{n} n \iff \\
        r &= \frac{1}{n} \mathbf{1}^T S(Y) \mathbf{1} .
    \end{align*}
    Combining this result with~\eqref{eq:mu-first-derivation}, we have
    \begin{equation}
        \label{eq:mu-second-derivation}
        \begin{aligned}
            \mu = \frac{2}{n} S(Y) \mathbf{1} - \frac{1}{n^2}  \left( \mathbf{1}^T S(Y) \mathbf{1} \right) \mathbf{1} ,
        \end{aligned}
    \end{equation}
    Plugging this result into~\eqref{eq:X-symmetric}, we obtain
    \begin{align*}
        X &= S(Y) - \frac{1}{2} \left( \frac{2}{n} S(Y) \mathbf{1} \mathbf{1}^T - \frac{1}{n^2}  \left( \mathbf{1}^T S(Y) \mathbf{1} \right) \mathbf{1} \mathbf{1}^T + \frac{2}{n}  \mathbf{1} \mathbf{1}^T S(Y) - \frac{1}{n^2} \mathbf{1} \mathbf{1}^T \left( \mathbf{1}^T S(Y) \mathbf{1} \right) \right) \\
        &= S(Y) - \frac{1}{n} S(Y) \mathbf{1} \mathbf{1}^T - \frac{1}{n} \mathbf{1} \mathbf{1}^T S(Y) + \frac{1}{n^2} \mathbf{1} \mathbf{1}^T S(Y) \mathbf{1} \mathbf{1}^T \\ 
        &= (I - \frac{1}{n} \mathbf{1} \mathbf{1}^T) S(Y) (I - \frac{1}{n} \mathbf{1} \mathbf{1}^T)
    \end{align*}
    This proves the lemma, and we arrive at the desired results
    \begin{align*}
        X = Q S(Y) Q , \quad \text{such that } Q = I - \frac{1}{n} \mathbf{1} \mathbf{1}^T .
    \end{align*}
\end{proof}

\section{Incorporating ASTRA subroutines in online algorithms}
\label{sec:ASTRA-incorporating}

As explained in Section~\ref{sec:online-derivation}, in this section, we show how the ASTRA subroutines, leveraging ADMM, are incorporated in the online algorithms, OGD and OCG. The online algorithms presented here follow the general structure presented in~\cite{hazan2016introduction}.

\paragraph*{Online Gradient Descent} Algorithm~\ref{alg:astra-ogd} summarizes the structure of OGD with the associated ADMM subroutine (Algorithm~\ref{alg:admm-based-projection-ogd}) incorporated. Here, the subroutine is responsible for projecting the gradient step with respect to the newly revealed cost function $f_t(X) = \frac{1}{2} \|(X - P_t) \odot U_t\|_F^2 - \mathbf{Tr}(\Lambda_t \odot X)$, where $X$ is constrained as in Problem~\eqref{eq:net-top-optimization-problem}, and $P_t, U_t, \Lambda_t$ denote the (possibly) time-varying matrices.

\begin{algorithm}[H]
\caption{ASTRA-OGD}
\label{alg:astra-ogd}
\begin{algorithmic}[1]
    \REQUIRE Initial point $X_1$, horizon $T \geq 1$, step-sizes $\eta_t$, for $t = 1, \dots, T$.
    
    \FOR{$t = 1, \dots, T$} 
    
        \STATE Output $X_t$
        
        \STATE Observe $f_t$ and accumulate loss $f_t(X_t)$

        \STATE $\overline{X}_t = X_t - \eta_t \nabla f_t(X_t)$
        
        \STATE $X_{t+1} =$ Algorithm~\ref{alg:admm-based-projection-ogd} with $\overline{X}_t$ as input
    \ENDFOR
\end{algorithmic}
\end{algorithm}

\paragraph*{Online Conditional Gradient} Algorithm~\ref{alg:astra-ocg}, summarizes the structure of OCG with the associated ADMM subroutine (Algorithm~\ref{alg:admm-based-linear-ocg}) incorporated, which is responsible for solving the linear optimization problem
\begin{align*}
    \underset{X}{\text{minimize}} &\quad \langle \nabla F_t(X_t), X \rangle_F \\
    \text{subject to} &\quad X=X^T , \quad X \mathbf{1} = \mathbf{0} , \quad B^{\text{lower}} \leq X \leq B^{\text{upper}} ,
\end{align*}
where $\nabla F_t(X_t)$ denotes the gradient of the function $F_t$ evaluated at the current estimate, where $F_t(X) \coloneqq \eta \sum_{\tau=1}^{t} \langle \nabla f_{\tau}(X_{\tau}) , X \rangle + \frac{1}{2} \|X - X_1\|_F^2$ is the cumulative sum of the cost functions until now plus a regularization term.

\begin{algorithm}[H]
\caption{ASTRA-OCG}
\label{alg:astra-ocg}
\begin{algorithmic}[1]
    \REQUIRE Initial point $X_1$, horizon $T \geq 1$, parameter $\eta$, $\sigma_t$, for $t = 1, \dots, T$.
    
    \FOR{$t = 1, \dots, T$} 
    
        \STATE Output $X_t$
        
        \STATE Observe $f_t$ and accumulate loss $f_t(X_t)$

        \STATE $F_t(X) = \eta \sum_{\tau=1}^{t} \langle \nabla f_{\tau}(X_{\tau}) , X \rangle + \frac{1}{2} \|X - X_1\|_F^2$

        \STATE $\nabla F_t(X_t) = \eta \sum_{\tau=1}^{t} \nabla f_{\tau}(X_{\tau}) + (X_t - X_1) $
        
        \STATE $R_t =$ Algorithm~\ref{alg:admm-based-linear-ocg} with $\nabla F_t(X_t)$ as input

        \STATE $X_{t+1} = (1 - \sigma_t) X_t + \sigma_t R_t$
    
    \ENDFOR
\end{algorithmic}
\end{algorithm}

\section{ASTRA projection for OGD}
\label{sec:admm-based-projection-for-ogd}

The projection of the gradient step $\overline{X}_t$ in Algorithm~\ref{alg:astra-ogd} onto, the set $K \cap B$, where $K \coloneqq \{X : X = X^T, X \mathbf{1} = \mathbf{0} \}$ and $B = \{X : B^{\text{lower}} \leq X \leq B^{\text{upper}} \}$ can be seen as solving the convex optimization problem
\begin{equation}
    \label{eq:projection-K-opt-problem}
    \begin{aligned}
        \underset{X}{\text{minimize}} &\quad \frac{1}{2} \left\|X - \overline{X}_t \right\|^2_F \\
        \text{subject to} &\quad X=X^T , \quad X \mathbf{1} = \mathbf{0}, \quad B^{\text{lower}}\le X \le B^{\text{upper}} .
    \end{aligned}
\end{equation}
Similarly to the derivation in Section~\ref{sec:admm-based-efficient-implementation}, we can express Problem~\eqref{eq:projection-K-opt-problem} in ADMM form as
\begin{equation}
    \begin{aligned}
        \underset{X, Z}{\text{minimize}} \quad f(X) + g(Z) \quad \text{subject to} \quad X - Z = \mathbf{0_{n \times n}} ,
    \end{aligned}
\end{equation}
where $\mathbf{0_{n \times n}}$ denotes the all-zeros $n \times n$ matrix, $g(Z) = i_{B}(Z)$ is the indicator function of set \(B \coloneqq \{X : B^{\text{lower}} \leq X \leq B^{\text{upper}} \}\), but
\begin{align*}
    &f(X) = \frac{1}{2} \left\|X - \overline{X}_t \right\|^2_F , \quad \text{dom} f = \{X : X = X^T, X \mathbf{1} = \mathbf{0} \} .
\end{align*}
Thus, as in Section~\ref{sec:admm-based-efficient-implementation}, the augmented Lagrangian (in the scaled form~\citep{boyd2011admm}) is given by $L_{\rho}(X,Z,V) = f(X) + g(Z) + \frac{\rho}{2} \left\| X - Z + V \right\|_F^2$, where $V$ is the scaled dual variable, and the general formulation of ADMM for this problem is given by
\begin{equation}
    \begin{aligned}
        X^{k+1} &= \underset{X \in \text{dom} f}{\arg\min} \left\{ f(X) + \frac{\rho}{2} \left\| X - Z^k + V^k \right\|_F^2 \right\} \\
        Z^{k+1} &= \underset{Z \in \mathbb{R}^{n \times n}}{\arg\min} \left\{ g(Z) + \frac{\rho}{2} \left\| X^{k+1} - Z + V^k \right\|_F^2 \right\} \\
        V^{k+1} &= V^k + \left( X^{k+1} - Z^{k+1} \right)
    \end{aligned}
\end{equation}

\paragraph*{Update over the $X$ variable.} The update over $X$ can be expressed as
\begin{equation}
    \label{eq:online-X-update}
    \begin{aligned}
        \underset{X}{\text{minimize}} &\quad \phi(X) \coloneqq \frac{1}{2} \left\|X - \overline{X}_t \right\|^2_F + \frac{\rho}{2} \left\| X - Z^k + V^k \right\|_F^2 \\
        \text{subject to} &\quad X=X^T \\
        &\quad X \mathbf{1} = \mathbf{0} .
    \end{aligned}
\end{equation}
Expanding the function $\phi$, we have
\begin{align*}
    \phi(X) = \frac{1}{2} \left[ (1 + \rho) \|X\|_F^2 - 2 \langle X, \overline{X}_t + \rho (Z^k - V^k) \rangle + C_1 \right] ,
\end{align*}
where $C_1 = \|\overline{X}_t\|_F^2 + \rho \|Z^k - V^k\|_F^2$ is a constant term.
Now, consider for a matrix $W$ the quadratic function
\begin{align*}
    (1 + \rho) \| X - W \|_F^2 = (1+\rho) \|X\|_F^2 - 2 (1+\rho) \langle X, W \rangle + (1+\rho) \|W\|_F^2
\end{align*}
Letting $W = \frac{\overline{X}_t + \rho (Z^k - V^k)}{1 + \rho}$, we have
\begin{align*}
    (1 + \rho) \| X - W \|_F^2 = (1+\rho) \|X\|_F^2 - 2 \langle X, \overline{X}_t + \rho (Z^k - V^k) \rangle + C_2
\end{align*}
where $C_2 = \frac{1}{1+\rho} \|\overline{X}_t + \rho (Z^k - V^k)\|_F^2$. Therefore, we see that we can rewrite $\phi$ as
\begin{align*}
    \phi(X) = \frac{1 + \rho}{2} \| X - W \|_F^2 + h(C_1,C_2),
\end{align*}
where $h$ is a function of the constant terms $C_1$ and $C_2$. Since the minimizer is not affected by the addition of constant terms, we see that Problem~\eqref{eq:online-X-update} reduces to the projection of matrix $W = \frac{\overline{X}_t + \rho (Z^k - V^k)}{1 + \rho}$ onto set $K = \{X : X = X^T , X \mathbf{1} = \mathbf{0} \}$, which we know has a closed-form solution by Lemma~\ref{lemma:projection-equalities}. Since $g$ is the indicator function of set \(B \coloneqq \{X : B^{\text{lower}} \leq X \leq B^{\text{upper}} \}\), the update of the $Z$ variable is given by Eq.~\eqref{eq:matrix-box-projection}, thus we arrive at the final form of Algorithm~\ref{alg:admm-based-projection-ogd}.

\section{ASTRA linear optimization for OCG}
\label{sec:admm-based-solution-ocg-linear-opt}

The update step in Algorithm~\ref{alg:astra-ocg} to obtain $R_t$ can be formulated as the following optimization problem
\begin{equation}
    \label{eq:linear-opt-problem-ocg}
    \begin{aligned}
            \underset{X}{\text{minimize}} &\quad \langle \nabla F(X_t), X \rangle_F \\
            \text{subject to} &\quad X=X^T, \quad X \mathbf{1} = \mathbf{0} , \quad B^{\text{lower}} \leq X \leq B^{\text{upper}} .
    \end{aligned}
\end{equation}
Similarly to the derivation in Section~\ref{sec:admm-based-efficient-implementation}, we can express Problem~\eqref{eq:linear-opt-problem-ocg} in ADMM form as
\begin{equation}
    \begin{aligned}
        \underset{X, Z}{\text{minimize}} \quad f(X) + g(Z) \quad \text{subject to} \quad X - Z = \mathbf{0_{n \times n}} ,
    \end{aligned}
\end{equation}
where $\mathbf{0_{n \times n}}$ denotes the all-zeros $n \times n$ matrix, $g(Z) = i_{B}(Z)$ is the indicator function of set \(B \coloneqq \{X : B^{\text{lower}} \leq X \leq B^{\text{upper}} \}\), but
\begin{align*}
    &f(X) = \langle \nabla F(X_t), X \rangle , \quad \text{dom} f = \{X : X = X^T, X \mathbf{1} = \mathbf{0} \} .
\end{align*}
Thus, as in Section~\ref{sec:admm-based-efficient-implementation}, the augmented Lagrangian (in the scaled form~\citep{boyd2011admm}) is given by $L_{\rho}(X,Z,V) = f(X) + g(Z) + \frac{\rho}{2} \left\| X - Z + V \right\|_F^2$, where $V$ is the scaled dual variable, and the general formulation of ADMM for this problem is given by
\begin{equation}
    \begin{aligned}
        X^{k+1} &= \underset{X \in \text{dom} f}{\arg\min} \left\{ f(X) + \frac{\rho}{2} \left\| X - Z^k + V^k \right\|_F^2 \right\} \\
        Z^{k+1} &= \underset{Z \in \mathbb{R}^{n \times n}}{\arg\min} \left\{ g(Z) + \frac{\rho}{2} \left\| X^{k+1} - Z + V^k \right\|_F^2 \right\} \\
        V^{k+1} &= V^k + \left( X^{k+1} - Z^{k+1} \right)
    \end{aligned}
\end{equation}
\paragraph*{Update over the $X$ variable.} The update over $X$ becomes

\begin{equation}
    \label{eq:online-X-update-ocg}
    \begin{aligned}
        \underset{X}{\text{minimize}} &\quad \phi(X) \coloneqq \langle \nabla F(X_t), X \rangle + \frac{\rho}{2} \left\| X - Z^k + V^k \right\|_F^2 \\
        \text{subject to} &\quad X=X^T, \quad X \mathbf{1} = \mathbf{0}
    \end{aligned}
\end{equation}
Expanding the function $\phi$, we have
\begin{align*}
    \phi(X) = \frac{1}{2} \left[ \rho \|X\|_F^2 - 2 \left\langle X, \rho (Z^k - V^k) - \nabla F_t(X_t) \right\rangle + C_1 \right]
\end{align*}
where $C_1 = \rho \|Z^k - V^k\|_F^2$ is a constant term. As in Appendix~\ref{sec:admm-based-projection-for-ogd}, consider for a matrix $W$ the quadratic function
\begin{align*}
    \rho \| X - W \|_F^2 = \rho \|X\|_F^2 - 2 \rho \langle X, W \rangle + \rho \|W\|_F^2
\end{align*}
Letting $$W = \frac{\rho (Z^k - V^k) - \nabla F_t(X_t)}{\rho} = (Z^k - V^k) - \frac{\nabla F_t(X_t)}{\rho} ,$$ we have
\begin{align*}
    \rho \| X - W \|_F^2 = \rho \|X\|_F^2 - 2 \left\langle X, \rho (Z^k - V^k) -  \nabla F_t(X_t) \right\rangle + C_2
\end{align*}
where $C_2 = \rho \left\| (Z^k - V^k) - \frac{\nabla F_t(X_t)}{\rho} \right\|_F^2$. Therefore, similarly to the solution for Online Gradient Descent, we see that we can rewrite $\phi$ as
\begin{align*}
    \phi(X) = \frac{\rho}{2} \| X - W \|_F^2 + h(C_1,C_2),
\end{align*}
where $h$ is a function of the constant terms $C_1$ and $C_2$. Since the minimizer is not affected by the addition of constant terms, we see that Problem~\eqref{eq:online-X-update-ocg} reduces to the projection of matrix $W = (Z^k - V^k) - \frac{\nabla F_t(X_t)}{\rho}$ onto set $K = \{X : X = X^T , X \mathbf{1} = \mathbf{0} \}$, which we know has a closed-form solution by Lemma~\ref{lemma:projection-equalities}. Since $g$ is the indicator function of set \(B \coloneqq \{X : B^{\text{lower}} \leq X \leq B^{\text{upper}} \}\), the update of the $Z$ variable is given by Eq.~\eqref{eq:matrix-box-projection}, thus we arrive at the final form of Algorithm~\ref{alg:admm-based-linear-ocg}.

\section{Dynamic regret with infeasible inner-loop iterates via feasible shadows}
\label{sec:dynamic-regret-feasible-shadows}

Before proving the dynamic regret bounds for infeasible inner-loop iterates, we will start by defining some notation and assumptions. Consider the sets $B \coloneqq \{X \in \mathbb{R}^{n \times n} : B^{\text{lower}} \leq X \leq B^{\text{upper}} \}$ and $K \coloneqq \{X \in \mathbb{R}^{n \times n} : X = X^T, X \mathbf{1} = \mathbf{0} \}$. Therefore, $\mathcal C := K \cap B \subset \mathbb R^{n\times n}$ denotes the feasible set of Problem~\eqref{eq:net-top-optimization-problem}. Furthermore, as introduced in Section~\ref{sec:oco_introd}, $P(U_1, \dots, U_T) \coloneqq \sum_{t=1}^{T-1} \|U_t - U_{t+1}\|_F + 1$ denotes the \emph{path length} of the dynamic comparator sequence $U_1, \dots, U_T$.
\begin{assumption}
    \label{assumption:diameter}
    Let $\mathcal C \coloneqq K \cap B \subset \mathbb R^{n\times n}$ be nonempty, convex, and compact. Therefore, there exists $D_{\mathcal C} \coloneqq \sup_{X,Z \in \mathcal C} \|X-Z\|_F$, which denotes the diameter with respect to the Frobenius norm.
\end{assumption}
\begin{assumption}
    \label{assumption:bounded-gradient}
    At each round $t$, the cost function $f_t : \mathbb{R}^{n \times n} \to \mathbb{R}$ is convex and differentiable, and satisfies $\| \nabla f_t(X) \|_F \leq L$, $\forall X \in \mathbb{R}^{n \times n}$.
\end{assumption}
Note that Assumption~\ref{assumption:bounded-gradient} implies that $| f_t( X ) - f_t(Y) | \leq L \left\| X - Y \right\|$, for all $X, Y \in \mathbb{R}^{n \times n}$ and $t \geq 1$ (as a consequence, e.g.,  of the mean-value theorem \citep[Theorem 2.3.3]{HiriartUrruty1993}).

\begin{lemma}[One-step inexact-projection inequality]
\label{lemma:one-step-inexact-projection}
Assume that Assumptions~\ref{assumption:diameter} and~\ref{assumption:bounded-gradient} hold. Let $G_t := \nabla f_t(\bar W_t)$ denote the gradient of $f_t$ evaluated at point $\bar W_t \in \mathcal{C}$, and define
$$
Y_t := \bar W_t - \eta_t G_t,
\qquad
P_t := \Pi_{\mathcal C}(Y_t),
$$
where $\Pi_{\mathcal C}(\cdot)$ is the unique projection onto $\mathcal C$.
Assume that the implemented update returns a feasible point $\bar W_{t+1} \in \mathcal C$ such that
\begin{equation}
    \label{eq:inexact-assumption}
    \begin{aligned}
        \|\bar W_{t+1}-P_t\|_F \le e_t .
    \end{aligned}
\end{equation}
Then, for any comparator $U_t \in \mathcal C$,
\begin{equation}
    \label{eq:inexact-feasible-instantaneous-regret}
    \begin{aligned}
        f_t(\bar W_t)-f_t(U_t)
        \leq
        \frac{\|\bar W_t-U_t\|_F^2-\|\bar W_{t+1}-U_t\|_F^2}{2\eta_t}
        +
        \frac{\eta_t}{2}\|G_t\|_F^2
        +
        \frac{D_{\mathcal C}}{\eta_t}e_t
        +
        \frac{e_t^2}{2\eta_t}.
    \end{aligned}
\end{equation}
\end{lemma}

\begin{proof}
By convexity of $f_t$, for any $U_t \in \mathcal C$,
$$
f_t(\bar W_t)-f_t(U_t)
\le
\langle G_t, \bar W_t-U_t\rangle_F.
$$
Thus, it suffices to upper-bound the Frobenius inner product on the right. Since $P_t = \Pi_{\mathcal C}(Y_t)$ is the Euclidean projection of $Y_t$ onto $\mathcal C$, and
$U_t \in \mathcal C$, we have
\begin{equation}
    \label{eq:non-expansiveness-projection}
    \begin{aligned}
        \|P_t-U_t\|_F \le \|Y_t-U_t\|_F.
    \end{aligned}
\end{equation}
Using the inexactness assumption in~\eqref{eq:inexact-assumption} and the triangle inequality,
$$
\|\bar W_{t+1}-U_t\|_F
\le
\|P_t-U_t\|_F + \|\bar W_{t+1}-P_t\|_F
\le
\|P_t-U_t\|_F + e_t.
$$
Squaring both sides yields
$$
\|\bar W_{t+1}-U_t\|_F^2
\le
\|P_t-U_t\|_F^2 + 2e_t\|P_t-U_t\|_F + e_t^2.
$$
Because $\bar W_{t+1},P_t,U_t \in \mathcal C$ and $\mathcal C$ has diameter $D_{\mathcal C}$ by Assumption~\ref{assumption:diameter}, we have
$\|P_t-U_t\|_F \le D_{\mathcal C}$, hence
\begin{equation}
    \label{eq:inexact-step-first-bound}
    \begin{aligned}
        \|\bar W_{t+1}-U_t\|_F^2
        \le
        \|P_t-U_t\|_F^2 + 2D_{\mathcal C}e_t + e_t^2
        \le
        \|Y_t-U_t\|_F^2 + 2D_{\mathcal C}e_t + e_t^2.
    \end{aligned}
\end{equation}
where the last inequality results from~\eqref{eq:non-expansiveness-projection}. Now, expand the square:
\begin{equation}
    \label{eq:non-expansive-square-expansion}
    \begin{aligned}
        \|Y_t-U_t\|_F^2
        =
        \|\bar W_t-\eta_t G_t-U_t\|_F^2
        =
        \|\bar W_t-U_t\|_F^2 - 2\eta_t\langle G_t,\bar W_t-U_t\rangle_F + \eta_t^2\|G_t\|_F^2.
    \end{aligned}
\end{equation}
Combining~\eqref{eq:inexact-step-first-bound} and~\ref{eq:non-expansive-square-expansion} gives
$$
\|\bar W_{t+1}-U_t\|_F^2
\le
\|\bar W_t-U_t\|_F^2
-2\eta_t\langle G_t,\bar W_t-U_t\rangle_F
+\eta_t^2\|G_t\|_F^2
+2D_{\mathcal C}e_t+e_t^2.
$$
Rearranging,
$$
\langle G_t,\bar W_t-U_t\rangle_F
\le
\frac{\|\bar W_t-U_t\|_F^2-\|\bar W_{t+1}-U_t\|_F^2}{2\eta_t}
+
\frac{\eta_t}{2}\|G_t\|_F^2
+
\frac{D_{\mathcal C}}{\eta_t}e_t
+
\frac{e_t^2}{2\eta_t}.
$$
Substituting this bound into the initial convexity inequality proves the claim.
\end{proof}

\begin{theorem}[Dynamic regret with infeasible inner-loop iterates via feasible shadows]
\label{thm:dyn-reg}
Assume that Assumptions~\ref{assumption:diameter} and~\ref{assumption:bounded-gradient} hold. Let $W_t \in \mathbb R^{n\times n}$ denote the possibly infeasible matrix played at round $t$, and define its feasible shadow
$$
\bar W_t := \Pi_{\mathcal C}(W_t).
$$
Fix a constant step size $\eta > 0$, and define the exact projected-gradient target
$$
Q_t := \Pi_{\mathcal C}\!\bigl(\bar W_t - \eta \nabla f_t(\bar W_t)\bigr).
$$
Let
$$
v_t := \|W_t-\bar W_t\|_F
=
\operatorname{dist}(W_t,\mathcal C)
$$
denote the infeasibility of the played iterate, and assume that the shadow sequence satisfies
$$
\|\bar W_{t+1}-Q_t\|_F \le e_t
\qquad
\forall t \in \{1,\dots,T\}.
$$
Then, for any comparator sequence $U_1,\dots,U_T \in \mathcal C$, the dynamic regret with respect to the sequence of possibly infeasible matrices $W_t$ is bounded by
$$
\sum_{t=1}^T \bigl(f_t(W_t)-f_t(U_t)\bigr)
\le
L \sum_{t=1}^T v_t
+
\frac{D_{\mathcal C}^2}{2\eta}
+
\frac{D_{\mathcal C}}{\eta} P(U_1, \dots, U_T)
+
\frac{\eta L^2T}{2}
+
\frac{D_{\mathcal C}}{\eta}\sum_{t=1}^T e_t
+
\frac{1}{2\eta}\sum_{t=1}^T e_t^2.
$$
\end{theorem}
\begin{proof}
For each $t$, decompose the regret term as
$$
f_t(W_t)-f_t(U_t)
=
\bigl(f_t(W_t)-f_t(\bar W_t)\bigr)
+
\bigl(f_t(\bar W_t)-f_t(U_t)\bigr).
$$
By Assumption~\ref{assumption:bounded-gradient}, $f_t(W_t)-f_t(\bar W_t) \leq L \|W_t-\bar W_t\|_F = L v_t$. Hence,
$$
f_t(W_t)-f_t(U_t)
\leq
L v_t + \bigl(f_t(\bar W_t)-f_t(U_t)\bigr).
$$
Summing from $t=1$ to $T$ yields
$$
\sum_{t=1}^T \bigl(f_t(W_t)-f_t(U_t)\bigr)
\le
L \sum_{t=1}^T v_t
+
\sum_{t=1}^T \bigl(f_t(\bar W_t)-f_t(U_t)\bigr).
$$
It therefore remains to bound the regret of the feasible shadow sequence $(\bar W_t)$. Since $\bar W_t,U_t,Q_t \in \mathcal C$ and
$$
Q_t = \Pi_{\mathcal C}\!\bigl(\bar W_t-\eta \nabla f_t(\bar W_t)\bigr),
\qquad
\|\bar W_{t+1}-Q_t\|_F \le e_t,
$$
Applying Lemma~\ref{lemma:one-step-inexact-projection} applied to the feasible sequence $(\bar W_t)$ gives
$$
f_t(\bar W_t)-f_t(U_t)
\le
\frac{\|\bar W_t-U_t\|_F^2-\|\bar W_{t+1}-U_t\|_F^2}{2\eta}
+
\frac{\eta}{2}\|\nabla f_t(\bar W_t)\|_F^2
+
\frac{D_{\mathcal C}}{\eta}e_t
+
\frac{e_t^2}{2\eta}.
$$
By Assumption~\ref{assumption:bounded-gradient}, $\|\nabla f_t(\bar W_t)\|_F \le L$, and summing from $t=1$ to $T$,
\begin{equation}
    \label{eq:second-term-regret}
    \begin{aligned}
        \sum_{t=1}^T \bigl(f_t(\bar W_t)-f_t(U_t)\bigr)
        &\le
        \frac{1}{2\eta}
        \sum_{t=1}^T
        \Bigl(
        \|\bar W_t-U_t\|_F^2-\|\bar W_{t+1}-U_t\|_F^2
        \Bigr)
        +
        \frac{\eta L^2T}{2} \\
        &\qquad
        +
        \frac{D_{\mathcal C}}{\eta}\sum_{t=1}^T e_t
        +
        \frac{1}{2\eta}\sum_{t=1}^T e_t^2.
    \end{aligned}
\end{equation}
To control the telescoping term with moving comparators, write
$$
\bar W_{t+1}-U_t = (\bar W_{t+1}-U_{t+1}) + (U_{t+1}-U_t).
$$
Then
\begin{align*}
\|\bar W_{t+1}-U_t\|_F^2
&=
\|\bar W_{t+1}-U_{t+1}\|_F^2
+
2\langle \bar W_{t+1}-U_{t+1},\,U_{t+1}-U_t\rangle_F
+
\|U_{t+1}-U_t\|_F^2 \\
&\ge
\|\bar W_{t+1}-U_{t+1}\|_F^2
-
2\|\bar W_{t+1}-U_{t+1}\|_F\,\|U_{t+1}-U_t\|_F \\
&\ge
\|\bar W_{t+1}-U_{t+1}\|_F^2
-
2D_{\mathcal C}\|U_{t+1}-U_t\|_F,
\end{align*}
because $\bar W_{t+1},U_{t+1}\in \mathcal C$.
Therefore,
$$
\|\bar W_t-U_t\|_F^2-\|\bar W_{t+1}-U_t\|_F^2
\le
\|\bar W_t-U_t\|_F^2-\|\bar W_{t+1}-U_{t+1}\|_F^2
+
2D_{\mathcal C}\|U_{t+1}-U_t\|_F.
$$
Summing this inequality from $t=1$ to $T$ and dropping the final nonpositive term gives
\begin{align*}
    \sum_{t=1}^T
    \Bigl(
    \|\bar W_t-U_t\|_F^2-\|\bar W_{t+1}-U_t\|_F^2
    \Bigr)
    &\le
    \|\bar W_1-U_1\|_F^2
    +
    2D_{\mathcal C}\sum_{t=1}^{T-1}\|U_{t+1}-U_t\|_F \\
    &\leq D_{\mathcal C}^2
    +
    2D_{\mathcal C} \, P(U_1, \dots, U_T)
\end{align*}
where the last inequality results from Assumption~\ref{assumption:diameter}. Substituting back in~\eqref{eq:second-term-regret} yields
$$
\sum_{t=1}^T \bigl(f_t(\bar W_t)-f_t(U_t)\bigr)
\le
\frac{D_{\mathcal C}^2}{2\eta}
+
\frac{D_{\mathcal C}}{\eta} P(U_1, \dots, U_T)
+
\frac{\eta L^2T}{2}
+
\frac{D_{\mathcal C}}{\eta}\sum_{t=1}^T e_t
+
\frac{1}{2\eta}\sum_{t=1}^T e_t^2.
$$
Combining this with the earlier decomposition proves the claim.
\end{proof}

\begin{corollary}[ADMM iterate corollary]
\label{cor:addm-iterate-corollary}
Assume the setting of Theorem~\ref{thm:dyn-reg}, and suppose that,
at round $t$, the inner ADMM loop returns matrices
$$
X_t^{k} \in K,
\qquad
Z_t^{k} \in B,
$$
after $k$ inner iterations. Define the played iterate as
$$
W_t := \lambda X_t^{k} + (1 - \lambda) Z_t^{k},
$$
a convex combination of $X_t^{k}$ and $Z_t^{k}$ for any $\lambda \in [0,1]$. Let $r_t^{k} := X_t^{k} - Z_t^{k}$ denote the ADMM primal residual. Furthermore, assume that the pair $(K,B)$ is boundedly linearly regular on a bounded set
containing the iterates, i.e., there exists $\kappa > 0$ such that
$$
\operatorname{dist}(W, K \cap B)
\le
\kappa\bigl(\operatorname{dist}(W,K) + \operatorname{dist}(W,B)\bigr)
$$
for all such $W$.\footnote{The set $K$ is affine and $B$ is a box, hence both are polyhedral. By applying the set-inclusion Hoffman bound of~\cite[Theorem 9]{burke1996hoffman} to the linear map $T(W) =(W,W)$ to the polyhedral set $K \times B$, equipped with the product norm $\|(U,V)\|=\|U\|+\|V\|$. Because the $T^{-1}(K \times B) = K \cap B$ and the the product norm is chosen so that the residual decomposes as the sum of the two distances, it holds that $\operatorname{dist}(W, K \cap B)
\le
\kappa\bigl(\operatorname{dist}(W,K) + \operatorname{dist}(W,B)\bigr)$.}
Then, for every $t$,
$$
v_t
=
\operatorname{dist}(W_t,\mathcal C)
\le
\kappa \|r_t^{k}\|_F.
$$
Consequently, if the shadow errors satisfy
$$
\|\bar W_{t+1} - \Pi_{\mathcal C}\!\bigl(\bar W_t - \eta \nabla f_t(\bar W_t)\bigr)\|_F
\le
e_t,
$$
then for any comparator sequence $U_1,\dots,U_T \in \mathcal C$,
\begin{align*}
    \sum_{t=1}^T \bigl(f_t(W_t)-f_t(U_t)\bigr)
    &\le
    L\kappa\sum_{t=1}^T \|r_t^{k}\|_F
    +
    \frac{D_{\mathcal C}^2}{2\eta}
    +
    \frac{D_{\mathcal C}}{\eta} P(U_1, \dots, U_T)
    +
    \frac{\eta L^2T}{2}
    \\
    &\quad+
    \frac{D_{\mathcal C}}{\eta}\sum_{t=1}^T e_t
    +
    \frac{1}{2\eta}\sum_{t=1}^T e_t^2.
\end{align*}
\end{corollary}
\begin{proof}
Since $X_t^{k} \in K$, we have
\begin{align*}
\operatorname{dist}(W_t,K)
\le
\|W_t - X_t^{k}\|_F
=
\|\lambda X_t^{k} + (1 - \lambda) Z_t^{k} - X_t^{k}\|_F
=
(1 - \lambda) \|r_t^{k}\|_F.
\end{align*}
Likewise, because $Z_t^{k} \in B$,
$$
\operatorname{dist}(W_t,B)
\le
\|W_t - Z_t^{k}\|_F
=
\|\lambda X_t^{k} + (1 - \lambda) Z_t^{k} - Z_t^{k}\|_F
=
\lambda \|r_t^{k}\|_F.
$$
By bounded linear regularity of $(K,B)$,
$$
v_t
=
\operatorname{dist}(W_t,K\cap B)
\le
\kappa\bigl(\operatorname{dist}(W_t,K) + \operatorname{dist}(W_t,B)\bigr)
\le
\kappa \|r_t^{k}\|_F.
$$
Substituting this bound for $v_t$ into Theorem~\ref{thm:dyn-reg}
yields the result.
\end{proof}

\section{Summary of ASTRA parameters}
\label{sec:astra-parameters}

In this section, we present a summary of the parameters that the proposed framework has, and their meaning, for both offline and online settings.

\input{Tables/parameters_ASTRA}

\section{Ablation Studies}
\label{app:ablation_studies}

In this section, we test the impact of multiple parameters on the results of the proposed method, by applying it to an Iridium-like constellation.

\subsection{Utility Weights - \(\lambda\)}

We first test the impact of parameter \(\lambda\) (see Sec.~\ref{sec:results}) on the topologies proposed by our ADMM method, using synthetic and real-world data. For this ablation, we fix the following parameters: \(\rho = 1\), \(\tau = -0.5\), \(\epsilon = 1\times10^{-7}\), and \(\text{proj}_{its} = 20\).

\input{Tables/iridium-lambda}

Tab.~\ref{tab:iridium-lambda-app} shows the results obtained for a synthetic Iridium-like constellation. Overall, values of \(\lambda\) between 0.45 and 0.75 are more desirable, as they obtain higher connectivity and lower average shortest paths. In terms of wall-clock time, we see that increasing the value of \(\lambda\) leads to slower execution when considering the ``H+C'' approximation. Overall, for both approximations, a value of \(\lambda\) of 0.6 gives us the best performance, as it achieves the highest connectivity and lowest average shortest path.

%\input{Tables/starlink-lambda}

%For Starlink data, our algorithm encounters greater difficulty in identifying topologies that form a single connected component compared to a synthetic Iridium-like constellation, Tab.~\ref{tab:starlink-lambda-app}. This behavior is expected, as Starlink constitutes a significantly larger and more complex constellation, characterized by non-uniform orbital plane distributions and irregular inter-satellite spacing. Despite the discrepancy in size and complexity, som of the conclusions derived from the simpler setting remain applicable for this constellation. As we increase \(\lambda\), we decrease the connectivity of the network, as the k-closest orbital plane utility considers fewer satellites, and increase execution time. Inasmuch as satellites have a greater number of neighbors, the impact of the approximation of the FOV is dampened. Overall, the best value of \(\lambda\) is 0.75, as it leads to networks with one connected component, and lowest average shortest path.

\subsection{Penalty Parameter - \(\rho\)}

Given the best value for \(\lambda\), i.e., 0.6, we now evaluate the impact of the penalty parameter \(\rho\) (recall Section~\ref{sec:admm_introduction}) on the topologies proposed by our ADMM method. We set \(\lambda = 0.6\), \(\tau = -0.5\), \(\epsilon = 1\times10^{-7}\), and \(\text{proj}_{its} = 20\).

\input{Tables/iridium-rho}

In the alternating direction method of multipliers, \(\rho\) is a penalty parameter on the norm of the constraints. Hence, higher values of \(\rho\) focus more on minimizing the norm of the constraints, rather than the objective function, while smaller values focus on the objective function rather than the constraints. Time-wise (Tab.~\ref{tab:rho_iridium}), values between 0.25, and 1 have faster execution when considering the ``H+C'' approximation. Overall, the value of this parameter seems to not affect much the properties of the topologies generated, as the metrics are equal for the majority of the values.

\subsection{Projected Gradient Iterations - \(\text{proj}_{its}\)}

As aforementioned in Section~\ref{sec:admm-based-efficient-implementation}, our ADMM method has a nested Projected Gradient loop to perform the updates over the matrix \(X\). Thus, in this section we evaluate the impact of the number of iterations of this loop in the topologies proposed by our method. We set \(\rho = 1\), \(\lambda = 0.6\), \(\tau = -0.5\) and  \(\epsilon = 1\times10^{-7}\). In this ablation, only the ``H'' approximation of the FOV of a satellite is considered.

\input{Tables/proj_its}

For this constellation and precision, we clearly see that the number of iterations of the Projected Gradient does not affect the characteristics of the proposed topologies. Since the precision is relatively small, this constellation is small, and follows full-deployment and uniform-spacing assumptions, this parameter as no impact on the metrics.
 
\section{Time comparison under Equal precision}
\label{app:time-eps-equal}

In this section, we compare the performance of ASTRA and GUTO on the Iridium-like constellation. Since it is not possible to define identical stopping criteria for ASTRA and GUTO, we run both algorithms until convergence to ensure a fair comparison. Note that in this case, the objective function is strongly convex (recall Sec.~\ref{sec:opt-problem-formulation}).

\input{Tables/time_2e-14_eps_astra_guto}

Tab.~\ref{tab:time_astra_guto_eps_2e-14} shows that GUTO takes on average 64\% more execution time than ASTRA to achieve the same solution on the Iridium-like constellation, showing that ASTRA takes even less time to achieve the same solution in higher precision.

\input{Files/new_experiments_online_strl}

\section{Learning-based baseline comparison}
\label{app:learning-based_comp}

The learning-based method used was MA-IILP~\cite{pi2022dynamic}, as it is the most related learning-based algorithm to ASTRA. Because no public MA-IILP implementation was available, we independently implemented it from the published pseudocode using the canonical OpenAI MADDPG training semantics. We evaluated it over 76 consecutive 300-second epochs of the intact 66-satellite Iridium-like constellation (6 planes, 11 satellites per plane). MA-IILP assumes a fixed constellation and agent/action space but was not designed for satellite addition/removal, so we evaluated it in its intended setting. Because MA-IILP does not report its numerical reward coefficients, we evaluated three values of its switching-to-rate coefficient. These symbols are unrelated to ASTRA’s regularization and ADMM parameters. Values below are means and standard deviations over the 76 topology epochs. The parameters used for both ASTRA and GUTO are the same as in Section~\ref{sec:results}, used for an Iridium-like constellation.

\input{Tables/learning-based_comparison}

In this setting all methods maintain connectivity except for DoTD (Tab.~\ref{tab:learning-based-comp}). Across the tested MA-IILP coefficients, ASTRA produces denser topologies with shorter average paths, matching GUTO’s topology quality while using the structured solver. We interpret this comparison cautiously because MA-IILP optimizes throughput and switching with a learned policy, whereas ASTRA optimizes a utility-driven Laplacian objective. Nevertheless, the common graph metrics provide a direct topology-level comparison, and ASTRA requires no 75,000-episode offline RL training.

\section{Effect of discretizing continuous solutions}
\label{app:mip_regret_comp}

The current theorems apply to decisions in the continuous convex relaxation. They do not automatically transfer unchanged after thresholding to discrete ISLs. If \(R_\tau\) denotes rounding at threshold \(\tau\) and each loss is 
\(L\)-Lipschitz, the discretized regret is bounded by the continuous regret plus a rounding term:

\begin{equation*}
    \text{Reg}_T^{\text{disc}} \le \text{Reg}_T^{\text{cont}} + L \sum_{t=1}^T \left|\left| R_\tau(X_t) - X_t \right|\right|
\end{equation*}

Consequently, the asymptotic order is preserved only when the cumulative rounding error is correspondingly controlled. Unrestricted per-step rounding can grow linearly. To quantify the practical effect, we reformulated the discrete problem as a mixed-integer convex quadratic program and solved it with GUROBI~\cite{gurobi} at tolerance \(10^{-7}\). We thresholded ASTRA-OGD/OCG using the \(\tau\) specified in Appendix~\ref{app:ablation_studies}, discretized them, and compared against the GUROBI solutions (Fig.~\ref{fig:mip_regret}). As expected, discretization increases empirical regret but preserves the observed advantage of ASTRA-OGD over ASTRA-OCG.

\input{Figures_tex/mip_regret_comparisons}

\clearpage

\end{document}

%% file: Files/net-top-config-off.tex
\paragraph*{Update over the $X$ variable.} The update over $X$ is equivalent to solving the following problem:
\begin{equation}
    \label{eq:offline-X-update}
    \begin{aligned}
        \underset{X}{\text{minimize}} &\quad h(X) \coloneqq \frac{1}{2} \|(X-P) \odot U\|_F^2 - \mathbf{Tr}(\Lambda \odot X) + \frac{\rho}{2} \left\| X - Z^k + V^k \right\|_F^2 \\
        \text{subject to} &\quad X=X^T, \quad X \mathbf{1} = \mathbf{0} .
    \end{aligned}
\end{equation}
Note that the projection onto set $K = \{X : X = X^T , X \mathbf{1} = \mathbf{0} \}$ has closed-form, as demonstrated in the following lemma.
\begin{lemma}
    \label{lemma:projection-equalities}
    Let $K = \{X : X = X^T , X \mathbf{1} = \mathbf{0} \}$, $S(A) = \frac{1}{2} (A + A^T)$, and $\Pi_{K}(\cdot)$ denote the projector operator onto $K$. Then, for any square matrix $Y$, we have $\Pi_{K}(Y) = Q S(Y) Q$ such that $Q = I - \frac{1}{n} \mathbf{1} \mathbf{1}^T$.
\end{lemma}
The proof for this lemma is provided in Appendix~\ref{sec:lemma_proj}. With this result, we can easily solve Problem~\eqref{eq:offline-X-update} through an iterative solution such as projected gradient descent. Alg.~\ref{alg:projected-gd} summarizes the iterative solution that receives $X^k$ and outputs the next estimate $X^{k+1}$.

\begin{algorithm}
\caption{Projected Gradient Descent}
\label{alg:projected-gd}
\begin{algorithmic}[1]
    \REQUIRE $K = \{X : X = X^T , X \mathbf{1} = \mathbf{0} \}$, $X_0 = X^k$, number of iterations $\tau$, step-size $\eta$.
    
    \FOR{$t = 1, \dots, \tau$} 
    
        \STATE $\nabla h(X_{t-1}) \coloneqq U \odot U \odot (X_{t-1} - P) - \Lambda + \rho (X_{t-1} - Z^k + V^k)$
        
        \STATE $X_{t} = \Pi_{{\mathcal K}} \left( X_{t-1} - \eta_t \nabla h(X_{t-1})\right)$
    \ENDFOR
    \STATE Output $X_{\tau}$
\end{algorithmic}
\end{algorithm}
\paragraph*{Update over the $Z$ variable.} Since $g$ is the indicator function of set \(B \coloneqq \{X : B^{\text{lower}} \leq X \leq B^{\text{upper}} \}\), then the update simply becomes the projection of $J \coloneqq X^{k+1} + V^k$  onto set $B$. Thus, $\forall i,j$,
\begin{equation}
    \label{eq:matrix-box-projection}
    \begin{aligned}
    Z^{k+1}_{ij} = \min\left( \max\left(J_{ij}, B^{\text{lower}}_{ij} \right), B^{\text{upper}}_{ij} \right) .
\end{aligned}
\end{equation}
These results allow us to construct an efficient ADMM to solve Problem~\eqref{eq:net-top-optimization-problem} (Alg.~\ref{alg:astra-admm-offline}). Nevertheless, we see that ASTRA consists of two nested iterative algorithms (ADMM and the Projected Gradient Descent for the update of the $X$ variable). Next, we show that the online algorithms can leverage the ADMM to achieve lower computational complexity than the offline version of ASTRA.
\begin{algorithm}[H]
\caption{ASTRA}
\label{alg:astra-admm-offline}
\begin{algorithmic}[1]
    \REQUIRE initialize $X^1, Z^1, V^1, \rho$, and $k = 1$.
    
    \WHILE{some stopping criteria are not met} 
    
        \STATE $X^{k+1} =$ Algorithm~\ref{alg:projected-gd} with initial point $X^k$ \\
        \STATE $Z^{k+1} = \Pi_B(X^{k+1} + V^k)$ \qquad\qquad (Eq.~\eqref{eq:matrix-box-projection}) \\
        \STATE $V^{k+1} = V^k + \left( X^{k+1} - Z^{k+1} \right)$ \\
        \STATE $k = k+1$
    \ENDWHILE
\end{algorithmic}
\end{algorithm}

%% file: Files/net-top-config-on.tex
\subsection{Online solutions}
\label{sec:online-derivation}

As demonstrated by Norberto et al.~\cite{norberto2026online}, Problem~\eqref{eq:net-top-optimization-problem} easily allows for the application of online algorithms for the problem of dynamic network topology configuration. In the framework of OCO, we can assume that, at each round of computation $t$, the revealed cost function $f_t$ has the form $f_t(X) = \frac{1}{2} \|(X - P_t) \odot U_t\|_F^2 - \mathbf{Tr}(\Lambda_t \odot X)$, where $X$ is constrained as in Problem~\eqref{eq:net-top-optimization-problem}, and $P_t, U_t, \Lambda_t$ denote the (possibly) time-varying matrices. To deal with the constraints affecting the variable $X$, a similar analysis to the one in the previous section shows that we can leverage ADMM to perform the projection step in OGD (Algorithm~\ref{alg:admm-based-projection-ogd}) and solve the linear optimization problem in OCG (Algorithm~\ref{alg:admm-based-linear-ocg}). For simplicity, in this section, we only present the structure of these subroutines. 
Based on the general formulation of OGD and OCG presented in~\cite{hazan2016introduction}, in Appendix~\ref{sec:ASTRA-incorporating} we show how we incorporate these subroutines. Their derivations are accessible in Appendices~\ref{sec:admm-based-projection-for-ogd} and~\ref{sec:admm-based-solution-ocg-linear-opt}, respectively. In Algorithm~\ref{alg:admm-based-projection-ogd}, $\overline{X}_t$ denotes the gradient step performed by OGD with respect to the newly revealed cost function $f_t$. In Algorithm~\ref{alg:admm-based-linear-ocg}, $\nabla F_t(X_t)$ denotes the gradient of the function $F_t$ evaluated at the current estimate, where $F_t(X) \coloneqq \eta \sum_{\tau=1}^{t} \langle \nabla f_{\tau}(X_{\tau}) , X \rangle + \frac{1}{2} \|X - X_1\|_F^2$.
\begin{figure}[t]
\centering

\begin{minipage}[t]{0.48\textwidth}
\captionsetup{type=algorithm}
\caption{ASTRA projection for OGD}
\label{alg:admm-based-projection-ogd}
\begin{algorithmic}[1]
    \REQUIRE initialize $X^1, Z^1, V^1, \rho$, receive $\overline{X}_t$ as input and $k = 1$.
    
    \WHILE{some stopping criteria are not met} 
        \STATE $W^k = \frac{\overline{X}_t + \rho (Z^k - V^k)}{1 + \rho}$ \\
        \STATE $X^{k+1} = \Pi_{K}(W^k)$ \qquad\quad (Lemma~\ref{lemma:projection-equalities}) \\
        \STATE $Z^{k+1} = \Pi_B(X^{k+1} + V^k)$  (Eq.~\eqref{eq:matrix-box-projection}) \\
        \STATE $V^{k+1} = V^k + \left( X^{k+1} - Z^{k+1} \right)$ \\
        \STATE $k = k+1$
    \ENDWHILE
    \RETURN $X^k$
\end{algorithmic}
\end{minipage}
\hfill
\begin{minipage}[t]{0.48\textwidth}
\captionsetup{type=algorithm}
\caption{ASTRA linear optimization for OCG}
\label{alg:admm-based-linear-ocg}%
\begin{algorithmic}[1]
    \REQUIRE initialize $X^1, Z^1, V^1, \rho$, receive $\nabla F_t(X_t)$ as input and $k = 1$.
    
    \WHILE{some stopping criteria are not met} 
        \STATE $W^k = (Z^k - V^k) - \frac{\nabla F_t(X_t)}{\rho}$ \\
        \STATE $X^{k+1} = \Pi_{K}(W^k)$ \qquad\quad (Lemma~\ref{lemma:projection-equalities}) \\
        \STATE $Z^{k+1} = \Pi_B(X^{k+1} + V^k)$  (Eq.~\eqref{eq:matrix-box-projection}) \\
        \STATE $V^{k+1} = V^k + \left( X^{k+1} - Z^{k+1} \right)$ \\
        \STATE $k = k+1$
    \ENDWHILE
    \RETURN $X^k$
\end{algorithmic}
\end{minipage}

\end{figure}
We observe that the computational complexity of OGD and OCG decreases significantly relative to the offline solution (Algorithm~\ref{alg:astra-admm-offline}), as in the ADMM sub-routine, we can obtain the $X$ update in closed form, instead of an iterative solution.

Different works have studied convergence rates and regret bounds considering inexact updates~\cite{schmidt2011inexact,liang2016convergence,dixit2019inexact}. Note that the use of the ADMM sub-routine encompasses the use of inexact updates on both algorithms. In the next theorem, we show, for OGD in particular, the dynamic regret under inexact ADMM inner loops, and relate the regret degradation to infeasibility and ADMM primal residuals.
\begin{theorem}[Dynamic regret bounds for OGD with infeasible inner-loop iterates]
\label{thm:dynamic-regret-ogd-bounds-infeasible}
    Let $\mathcal C := K \cap B \subset \mathbb R^{n\times n}$ denote the feasible set of Problem~\eqref{eq:net-top-optimization-problem}, and $D_{\mathcal C} \coloneqq \sup_{X,Z \in \mathcal C} \|X-Z\|_F$ its diameter with respect to the Frobenius norm. Suppose that, at round $t$, the inner ADMM loop returns $X_t^{k} \in K$, after $k$ inner iterations. Define the played iterate as $W_t := X_t^{k}$, and let $r_t^{k} := X_t^{k} - Z_t^{k}$ denote the ADMM primal residual. Let $\bar W_t := \Pi_{\mathcal C}(W_t)$ denote the feasible shadow of $W_t$. Assuming the shadow errors satisfy
    $$
    \|\bar W_{t+1} - \Pi_{\mathcal C}\!\bigl(\bar W_t - \eta \nabla f_t(\bar W_t)\bigr)\|_F
    \le
    e_t.
    $$
    Furthermore, assume that the pair $(K,B)$ is boundedly linearly regular, for some $\kappa > 0$, on a bounded set containing the iterates. Then, for any comparator sequence $U_1, \dots, U_T \in \mathcal C$, the dynamic regret of OGD with step-size $\eta$ is bounded by
    \begin{align*}
    \sum_{t=1}^T \bigl(f_t(W_t)-f_t(U_t)\bigr)
    &\le
    L\kappa\sum_{t=1}^T \|r_t^{(k)}\|_F
    +
    \frac{D_{\mathcal C}^2}{2\eta}
    +
    \frac{D_{\mathcal C}}{\eta} P(U_1, \dots, U_T)
    +
    \frac{\eta L^2T}{2}
    \\
    &\quad+
    \frac{D_{\mathcal C}}{\eta}\sum_{t=1}^T e_t
    +
    \frac{1}{2\eta}\sum_{t=1}^T e_t^2.
    \end{align*}
\end{theorem}
Here, $P(U_1, \dots, U_T) \coloneqq \sum_{t=1}^{T-1} \|U_t - U_{t+1}\| + 1$ denotes the \emph{path length} of the dynamic comparator sequence $U_1, \dots, U_T$. This result is a direct consequence of Corollary~\ref{cor:addm-iterate-corollary} in Appendix~\ref{sec:dynamic-regret-feasible-shadows}, where we quantify a bound on dynamic regret for inexact update steps that arise from the precision of ADMM. Furthermore, if we have information about the path length of the dynamic comparator, then, as in~\cite[Theorem 10.1]{hazan2016introduction}, we can define $\eta \coloneqq \sqrt{\frac{P(U_1, \dots, U_T)}{T}}$ to obtain a dynamic regret bound that is sublinear on both the horizon $T$ and the path length $P(U_1, \dots, U_T)$. In the next section, we show that the path length can be estimated offline, enabling deployment of a fine-tuned optimal OGD for dynamic network topology configuration.

%% file: Tables/deg_dist_iridium.tex
\begin{table}
    \centering
    \scriptsize
    \caption{Average degree distribution of GUTO and ASTRA, for an Iridium-like constellation.}
    \label{tab:avg_degree_distribution}
    \begin{tabular}{ccccccc}
        \cmidrule{3-7}
        & & \multicolumn{5}{c}{Degree} \\
        \midrule
        Method: & Approximation & 0 & 1 & 2 & 3 & 4 \\
        \midrule
        GUTO  & \multirow{2}{*}{H}   & \(0 \pm 0\) & \(0 \pm 0\) & \(0.5 \pm 0.5\) & \(16 \pm 2\) & \(49.5 \pm 1.5\) \\
        ASTRA &                      & \(0 \pm 0\) & \(0 \pm 0\) & \(0.5 \pm 0.5\) & \(16 \pm 2\) & \(49.5 \pm 1.5\) \\
        \midrule
        GUTO  & \multirow{2}{*}{H+C} & \(0 \pm 0\) & \(0 \pm 0\) & \(0.5 \pm 0.5\) & \(12 \pm 4\) & \(53.5 \pm 4.5\) \\
        ASTRA &                      & \(0 \pm 0\) & \(0 \pm 0\) & \(0.5 \pm 0.5\) & \(12 \pm 4\) & \(53.5 \pm 4.5\) \\
        \bottomrule
    \end{tabular}
\end{table}

%% file: Tables/iridium_ours_sota.tex
\begin{table}[ht!]
\centering
\scriptsize
\caption{Methods applied on an Iridium-like constellation. ASTRA is capable of achieving the same results as GUTO in a fraction of the time, closely matching the state-of-the-art.}
\begin{tabular}{cccccc}
\hline
Metrics & E \(\uparrow\)           & A-Deg  \(\uparrow\)      & CC \(\downarrow\) & A-SP   \(\downarrow\)        & Time (s) \(\downarrow\)\\ \hline
+Grid   & 130    & 3.94   & 1  & {\ul 4.036}    & {\textbf{0.001}}    \\
\(\times\)Grid   & 124          & 3.76        & 1  & 4.978          & \textbf{0.001}    \\
DoTD    & {\ul 131} & {\ul 3.97} & 1  & 4.412          & 0.023    \\
Motif   & \textbf{132}      & \textbf{4.00}          & 1  & 4.292          & 1.326    \\
GUTO   & 128          & 3.88        & 1  & \textbf{3.767} &  0.091     \\ 
ASTRA   &128           &3.88         &1   & \textbf{3.767} & 0.071      \\ \hline
\end{tabular}
\label{tab:iridium-like_constellation}
\end{table}

%% file: Figures_tex/regret_comparison.tex
\begin{figure}[ht!]
    \centering
    %\scriptsize
    \includegraphics[width=.9\columnwidth]{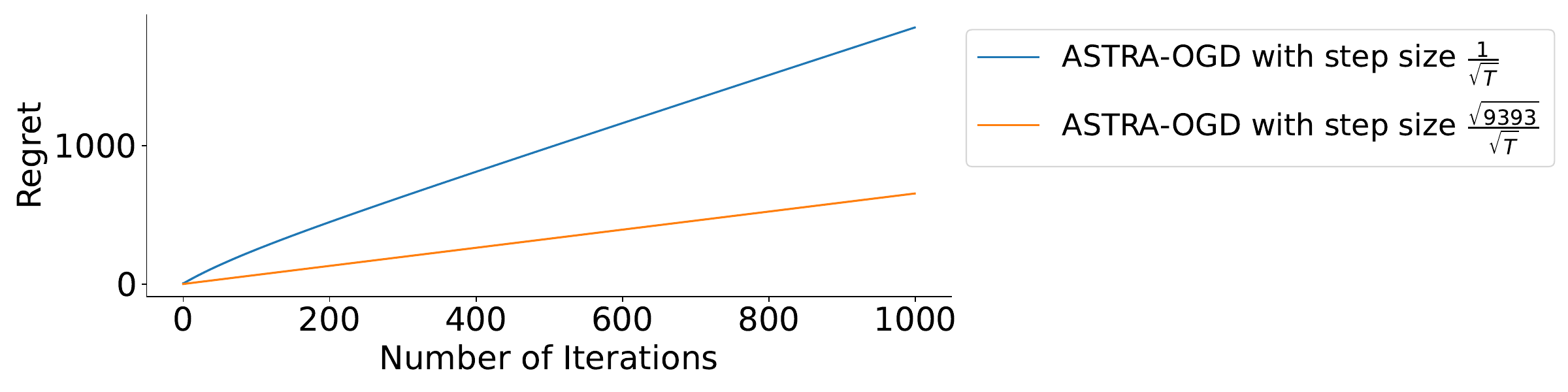}
	\caption{Evolution of the regret for ASTRA-OGD with the static and dynamic regret step sizes, on an Iridium-like constellation with 6 orbital planes and 11 satellites per plane, all uniformly spaced.}
	\label{fig:regret_comparison}
\end{figure}

%% file: Figures_tex/online_adapt_rem_add.tex
\begin{figure}[ht!]
    \centering
    \scriptsize
    \includegraphics[width=.9\columnwidth]{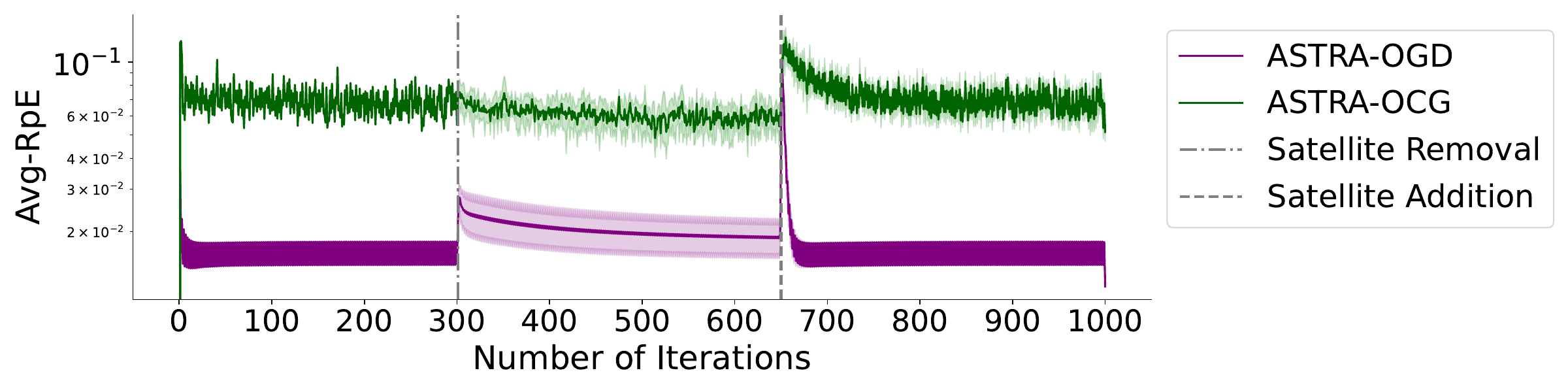}
	\caption{Smoothed average residual per entry (Avg-RpE) between the matrices proposed by ASTRA-OGD and ASTRA-OCG relative to GUTO. Six random satellites are removed at iteration 300 and restored to their original positions at iteration 650. ASTRA-OGD is capable of proposing similar solutions to GUTO, when satellites are added or removed, while ASTRA-OCG struggles to adapt.}
	\label{fig:adapt_online_rem_add}
\end{figure}

%% file: Tables/offline_starlink.tex
\begin{table}[ht!]
\centering
\scriptsize
\caption{Comparison on Starlink data. Only GUTO and ASTRA attain runs with a single connected component; therefore, A-SP* is reported only for topologies satisfying this condition. ASTRA preserves the connectivity regime of GUTO while improving efficiency.}

\begin{tabular}{@{}cccccc@{}}
\cmidrule(l){3-6}
                  &                      & \multicolumn{4}{c}{Metrics} \\ \midrule
Method:            & Approximation        & E\(\uparrow\)   & A-Deg\(\uparrow\)   & CC\(\downarrow\)  & A-SP* \(\downarrow\)  \\
\midrule
+Grid             & \multirow{3}{*}{---}                  &\(2869.14_{\pm14.123}\)&\(3.608_{\pm0.017}\)       &\(9.038_{\pm2.019}\)     &---       \\
\(\times\)Grid    &                   &\(2468.77_{\pm19.325}\)&\( 3.205_{\pm0.026}\)         &\(3.346_{\pm1.314}\)     &---       \\
DoTD              & &\(\mathbf{3183.16_{\pm0.586}}\)     &\(\mathbf{3.998_{\pm0.001}}\)         &\(4.211_{\pm1.436}\)     &---       \\ \midrule
GUTO        & \multirow{2}{*}{H}   & {\ul \(2900.95_{\pm13.033}\)}    &{\ul \(3.650_{\pm0.015}\)}         & \(\mathbf{1.233_{\pm0.491}}\)     &\(15.742_{\pm0.388}\)       \\
ASTRA &                      & \(2900.24_{\pm13.058}\)&\(3.649_{\pm0.015}\)& {\ul \(1.264_{\pm0.533}\)}     &\(15.760_{\pm0.381}\)    \\ \midrule
%ASTRA-OGD* &    &    &    &    &   \\\midrulec
GUTO        & \multirow{2}{*}{H+C} &\(2797.16_{\pm23.664}\)     &\(3.544_{\pm0.025}\)    & \(1.642_{\pm0.878}\)    &\(\mathbf{11.771_{\pm0.226}}\)       \\
ASTRA &                      &\(2795.63_{\pm23.506}\)&\(3.542_{\pm0.025}\)&\(1.646_{\pm0.878}\)& {\ul \(11.817_{\pm0.224}\)}     \\ \bottomrule
\end{tabular}
\label{tab:offline_starlink}
\end{table}

%% file: Tables/deg_dist_starlink.tex
\begin{table}
    \centering
    \scriptsize
    \caption{Average degree distribution of GUTO and ASTRA, for a Starlink constellation.}
    \label{tab:avg_degree_distribution_strl}
    \begin{tabular}{ccccccc}
        \cmidrule{3-7}
        & & \multicolumn{5}{c}{Degree} \\
        \midrule
        Method: & Approximation & 0 & 1 & 2 & 3 & 4 \\
        \midrule
        GUTO  & \multirow{2}{*}{H}   & \(1.61 \pm 1.47\) & \(10.54 \pm 3.47\) & \(81.14 \pm 8.20\) & \(363.79 \pm 16.11\) & \(1133.92 \pm 17.53\) \\
        ASTRA &                      & \(1.62 \pm 1.47\) & \(10.65 \pm 3.48\) & \(81.64 \pm 8.32\) & \(363.97 \pm 16.12\) & \(1133.13 \pm 17.49\) \\
        \midrule
        GUTO  & \multirow{2}{*}{H+C} & \(12.46 \pm 3.71\) & \(32.73 \pm 6.31\) & \(121.35 \pm 13.47\) & \(378.91 \pm 17.64\) & \(1045.56 \pm 24.79\) \\
        ASTRA &                      & \(12.46 \pm 3.74\) & \(33.06 \pm 6.26\) & \(122.32 \pm 12.48\) & \(379.23 \pm 17.59\) & \(1043.94 \pm 24.63\) \\
        \bottomrule
    \end{tabular}
\end{table}

%% file: Tables/times_offline.tex
\begin{table}[ht!]
\centering
\scriptsize
\caption{Average time and respective std each method takes to propose a topology (Starlink), after 10 runs.+Grid and \(\times\)Grid are the fastest methods.}
\begin{tabular}{@{}cccccc@{}}
\toprule
Method: & +Grid & \(\times\)Grid & DoTD & GUTO & ASTRA \\ \midrule
Time (s)   &\(0.021_{\pm0.001}\)       &\(0.026_{\pm0.001}\)   &\(3262.927_{\pm48.097}\)      &       \(665.027_{\pm2.093}\)&  \(567.949_{\pm5.615}\)    \\ \bottomrule
\end{tabular}
\label{tab:times_offline}
\end{table}

%Time each method takes to propose a topology (Starlink). This table shows the average wall-clock time and respective std of running each method 10 times. +Grid and \(\times\)Grid are the fastest methods.

%% file: Figures_tex/bar_logos.tex
\begin{figure}[H]
    \centering
    \scriptsize
    \includegraphics[width=.75\columnwidth]{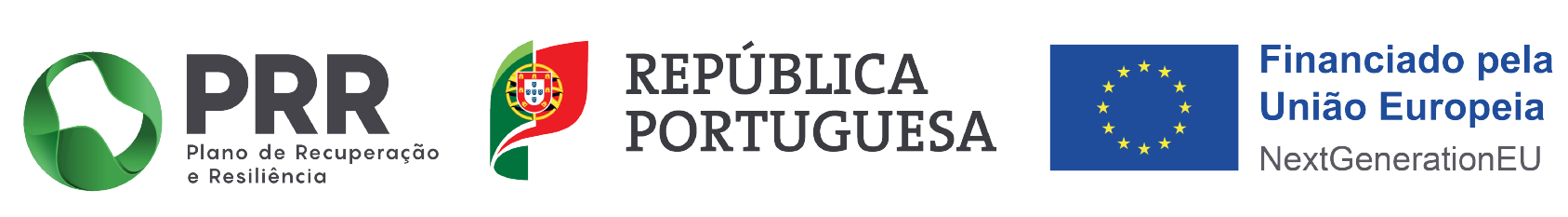}
	\label{fig:bar_logos}
\end{figure}

%% file: Tables/parameters_ASTRA.tex
\begin{table}[ht!]
\centering
\caption{List of the parameters for the ASTRA offline algorithm.}
\begin{tabular}{@{}cc@{}}
\toprule
Parameter:            & Description:                                                                                                                                \\ \midrule
\(\lambda\)           & Weight given to the utilities being considered.                                                                                             \\
\(\rho\)              & Penalty parameter from the Augmented Lagrangian.                                                                                            \\
\(\epsilon\)          & Stopping criteria value.                                                                                                                    \\
\(\text{proj}_{its}\) & \begin{tabular}[c]{@{}c@{}}Number of iterations for the \\ Projected Gradient (Alg.~\ref{alg:projected-gd}) that deals with \\ the update over the matrix \(X\).\end{tabular} \\
T                     & \begin{tabular}[c]{@{}c@{}}Maximum number of iterations of the ASTRA\\ algorithm.\end{tabular}                                              \\
\(\Lambda\)           & \begin{tabular}[c]{@{}c@{}}Weight to be given to the connections between\\ satellites.\end{tabular}                                         \\
FOV Approximation     & \begin{tabular}[c]{@{}c@{}}Which approximation of the FOV of a satellite\\ to consider.\end{tabular}                                        \\\bottomrule
\end{tabular}
\label{tab:parameter_table}
\end{table}

%% file: Tables/iridium-lambda.tex
\begin{table}[ht!]
\centering
\caption{Effect of the parameter \(\lambda\) and the approximation of the FOV of a satellite on the results of ASTRA, for an Iridium-like constellation.}
\begin{tabular}{@{}ccccccc@{}}
\cmidrule(l){3-7}
            &               & \multicolumn{5}{c}{Metrics}                         \\ \midrule
\(\lambda\) & Approximation & E\(\uparrow\) & A-Deg\(\uparrow\) & CC\(\downarrow\) & A-SP\(\downarrow\) & Time (s) \\ \midrule

0    & \multirow{7}{*}{H}   & 121 & 3.667 & 1 & 4.678 & 0.199  \\
0.15 &                        & 122 & 3.697 & 1 & 4.648 & 0.203  \\
0.30 &                        & 122 & 3.697 & 1 & 4.691 & 0.203  \\
0.45 &                        & 123 & 3.727 & 1 & 4.676 & 0.208  \\
0.60 &                        & 124 & 3.758 & 1 & 4.556 & 0.200  \\
0.75 &                        & 124  & 3.758 & 1 & 4.540 & 0.202  \\
0.90 &                        & 115  & 3.485 & 1 & 5.336   & 0.204  \\ \midrule
%1    &                        & 66  & 2.276 & 2 & ---   & 0.169  \\ \midrule

0    & \multirow{7}{*}{H+C} & 113 & 3.424 & 1 & 4.883 & 0.053 \\
0.15 &                        & 113 & 3.424 & 1 & 5.034 & 0.052 \\
0.30 &                        & 110 & 3.333 & 1 & 5.874 & 0.062 \\
0.45 &                        & 119 & 3.662 & 1 & 4.462 & 0.076 \\
0.60 &                        & 128 & 3.879 & 1 & 3.767 & 0.167 \\
0.75 &                        & 124 & 3.758 & 1 & 3.774 & 0.219 \\
0.90 &                        & 120 & 3.636 & 1 & 4.140 & 0.220 \\ \bottomrule
%1    &                        & 111 & 3.524 & 1 & 4.495 & 0.137 \\ \bottomrule

\end{tabular}
\label{tab:iridium-lambda-app}
\end{table}

%% file: Tables/iridium-rho.tex
\begin{table}[ht!]
\centering
\caption{Effect of the parameter \(\rho\) and the approximation of the FOV of a satellite on the results of ASTRA, for an Iridium-like constellation.}
\begin{tabular}{@{}ccccccc@{}}
\cmidrule(l){3-7}
            &               & \multicolumn{5}{c}{Metrics}                         \\ \midrule
\(\rho\)    & Approximation & E\(\uparrow\) & A-Deg\(\uparrow\) & CC\(\downarrow\) & A-SP\(\downarrow\) & Time (s) \\ \midrule
0.001 &  \multirow{9}{*}{H} &128   &3.758   &1   &4.531   &0.204     \\
0.01               &      &124   &3.758   &1   &4.556   &0.202      \\
0.1                &      &124   &3.758   &1   &4.556   &0.202     \\
0.25                &     &124   &3.758   &1   &4.556   &0.217     \\
0.5                &      &124   &3.758   &1   &4.556   &0.217      \\
0.75                &     &124   &3.758   &1   &4.556   &0.209      \\
1                  &      &124   &3.758   &1   &4.556   &0.203     \\
2                  &      &124   &3.758   &1   &4.556   &0.207      \\
4                  &      &124   &3.758   &1   &4.556   &0.216      \\ \midrule

0.001 &  \multirow{9}{*}{H+C} &127   &3.848   &1  &3.866   &0.211      \\
0.01               &      &126   &3.818   &1   &3.802   &0.220      \\
0.1                &      &128   &3.879   &1   &3.767   &0.175      \\
0.25                &      &128   &3.879   &1   &3.767   &0.077     \\
0.5                &      &128   &3.879   &1   &3.767   &0.104      \\
0.75                &      &128   &3.879   &1   &3.767   &0.140    \\
1                  &      &128   &3.879   &1   &3.767   &0.164    \\
2                  &      &128   &3.879   &1   &3.767   &0.218      \\
4                  &      &128   &3.879   &1   &3.767   &0.217      \\ \bottomrule
\end{tabular}
\label{tab:rho_iridium}
\end{table}

%% file: Tables/proj_its.tex
\begin{table}[ht!]
\centering
%\scriptsize
\caption{Effect of the parameter \(\text{proj}_{its}\) on the results of ASTRA, for an Iridium-like constellation.}
\begin{tabular}{@{}cccccc@{}}
\cmidrule(l){2-6}
                          & \multicolumn{5}{c}{Metrics}                         \\ \midrule
\(\text{proj}_{its}\)     & E\(\uparrow\) & A-Deg\(\uparrow\) & CC\(\downarrow\) & A-SP\(\downarrow\) & Time (s)\\ \midrule
5 &  124   &3.758   &1   &4.556   &0.075   \\
10                     &124   &3.758   &1   &4.556   &0.108     \\
20                      &124   &3.758   &1   &4.556   &0.207    \\
30                     &124   &3.758   &1  &4.556   &0.323     \\
40                      &124   &3.758   &1  &4.556   &0.412      \\
50                      &124   &3.758   &1  &4.556   &0.501      \\ 
60                      &124   &3.758   &1  &4.556   &0.594      \\ \bottomrule
\end{tabular}
\label{tab:projected_its}
\end{table}

%% file: Tables/time_2e-14_eps_astra_guto.tex
\begin{table}[]
\scriptsize
\centering
\caption{Time comparison when GUTO and ASTRA are run until achieving the same result. The values presented on this table refer to the average and respective standard deviation over 10 runs.}
\label{tab:time_astra_guto_eps_2e-14}
\begin{tabular}{@{}ccc@{}}
\toprule
Method: & GUTO & ASTRA \\ \midrule
Time (s) & \(0.991_{\pm 0.008}\) & \(\mathbf{0.356_{\pm 0.051}}\) \\ 
\bottomrule
\end{tabular}
\end{table}

%% file: Files/new_experiments_online_strl.tex
\section{Online Algorithms on Starlink data}
\label{app:on_alg_on_strl}

In this setting, we evaluate ASTRA-OGD and ASTRA-OCG (Section~\ref{sec:online-derivation}) on the first 70 Starlink topologies retrieved in Section~\ref{sec:results}. The parameters considered for ASTRA, GUTO, ASTRA-OGD and ASTRA-OCG are the same as in Section~\ref{sec:results}, used for Starlink data, apart from the approximation of the FOV, which we consider to be ``H''. The learning rate for the latter is as follows; ASTRA-OCG has a step-size \(\eta=\frac{1}{T^{\frac{3}{4}}}\), and \(\sigma=\text{min}\{1, \frac{2}{t^{\frac{1}{2}}}\}\)~\cite{hazan2016introduction}; ASTRA-OGD has a dynamic regret step size $\eta = \sqrt{\frac{P(U_1, \dots, U_T)}{T}} = \sqrt{\frac{3280.452}{T}}$, as it achieves smaller regret values than the standard step-size (Section~\ref{sec:results} - OGD with fine-tuned step-size for optimal dynamic regret).

\input{Tables/offline_vs_online_70its}

\input{Figures_tex/astra-ogd_vs_astra-ocg_70its}

Tab.~\ref{tab:online_vs_offline_strl_70its} shows the metrics for the proposed graphs, as well as the total number of connections that do not respect matrix \(P\) (Section~\ref{sec:opt-problem-formulation}), ``\(\#\{\neg P\}\)'', and its average value per topology, ``A-\(\#\{\neg P\}\)''. The connectivity of the topologies proposed by ASTRA-OGD sees a decrease in comparison with GUTO and ASTRA, while 3,25\% of the connections of these topologies do not conform with matrix \(P\). While we lose connectivity properties, compared to the offline methods, ASTRA-OGD remains a competitive adaptive solution for satellite network topology configuration, as it leads to significantly lower computational burden without the need to recompute the topology from scrath at each iteration, as will be elaborated in the remainder of the section. Although ASTRA-OCG seems to be the best performing method, as it achieves the highest connectivity metrics and the smallest ``SP*'', it achieves the highest value of \(\#\{\neg P\}\) and A-\(\#\{\neg P\}\), with an increase of 150\%, when compared with ASTRA-OGD. Thus, its metrics are achieved by ignoring the FOV and distance constraints expressed by matrix \(P\), for ISLs. Regret-wise, while ASTRA-OGD achieves a regret of \(2294.224\), ASTRA-OCG sees an increase of 640\%, obtaining a regret value of \(14659.08\) (see the figure on the right in Fig~\ref{fig:astra-ogd_vs_astra-ocg_its}).

We also see a discrepancy between the number of iterations each algorithm takes to converge (see the figure on the left in Fig~\ref{fig:astra-ogd_vs_astra-ocg_its}). ASTRA-OGD seems to maintain a similar number of iterations until convergence, around six to seven thousand, taking on average \(62.857_{\pm 6.058}\)(s), a 89\% faster execution than ASTRA. ASTRA-OCG shows a greater variance on the number of iterations, taking around five times more iterations than ASTRA-OGD to converge. It is also notable that after 20 topologies, ASTRA-OCG shows an increase in the number of iterations needed to converge, as more topologies are processed.

\input{Tables/astra-ocg_iterations_time}

Tab.~\ref{tab:astra-ocg-time-its} shows the wall-clock time that ASTRA-OCG takes to reach a certain number of iterations, for a fixed topology. Considering the best-case scenario, ASTRA-OCG takes around twenty thousand iterations to converge (Fig~\ref{fig:astra-ogd_vs_astra-ocg_its}), leading to a 42\% faster execution than ASTRA and 81\% slower execution when compared with ASTRA-OGD. As the number of iterations to converge increases, the wall-clock time also increases, achieving execution times greater than GUTO (Tab.~\ref{tab:astra-ocg-time-its}). Thus, there seems to exist a proportional relationship between the number of topologies processed by ASTRA-OCG and the time needed to propose new ones.

Altogether, we see a trade-off between efficiency and network connectivity when using ASTRA-OGD. Contrastingly, by considering and assigning the same importance to all the previous loss functions, ASTRA-OCG struggles to conform with the constraints on ISLs, established by the matrix \(P\), leading to the significant disparities in performance to the other methods.

%% file: Tables/offline_vs_online_70its.tex
\begin{table}[]
\scriptsize
\centering
\caption[Comparison of Online and Offline algorithms in Starlink data.]{Graph metric comparison between offline and online algorithms for 70 consecutive Starlink topologies. In addition to graph metrics, we also consider the total number of connections that do not respect matrix \(P\), ``\(\#\{\neg P\}\)'', and its average value per topology ``A-\(\#\{\neg P\}\)''.}
\begin{tabular}{@{}ccccccc@{}}
\cmidrule(l){2-7}
 & \multicolumn{6}{c}{Metrics} \\ \midrule
Method: & E \(\uparrow\)                       & A-Deg\(\uparrow\)                 & CC \(\downarrow\)                   & SP* \(\downarrow\) & \(\#\{\neg P\}\) \(\downarrow\) & A-\(\#\{\neg P\}\)  \(\downarrow\)                  \\ \midrule
GUTO      &  \(\mathbf{2898.71_{\pm 13.993}}\) & \(\mathbf{3.647_{\pm 0.017}}\) & \(\mathbf{1.129_{\pm 0.335}}\) & \(\mathbf{15.668_{\pm 0.372}}\)    & 9 & \(0.13_{\pm0.448}\) \\
ASTRA     &  \(2897.87_{\pm 14.232}\) & \(3.646_{\pm 0.017}\) & \(1.171_{\pm 0.377}\) & \(15.704_{\pm 0.355}\) & 0 & \(0_{\pm0.0}\)  \\ \midrule
ASTRA-OGD & \(2746.68_{\pm 19.604}\) & \(3.456_{\pm 0.024}\) & \(1.855_{\pm 0.728}\) & \(14.303_{\pm 0.431}\) & \textbf{6153} & \(\mathbf{89.174_{\pm 9.902}}\) \\
ASTRA-OCG & \(\mathbf{3169.20_{\pm 2.511}}\)  & \(\mathbf{3.984_{\pm 0.003}}\) & \(\mathbf{1.029_{\pm 0.168}}\) & \(\mathbf{11.103_{\pm 0.473}}\) & 15199 & \(220.275_{\pm 48.962}\) \\ \bottomrule
\end{tabular}
\label{tab:online_vs_offline_strl_70its}
\end{table}

%% file: Figures_tex/astra-ogd_vs_astra-ocg_70its.tex
\begin{figure}[h!]
    \centering
    \includegraphics[width=1\linewidth]{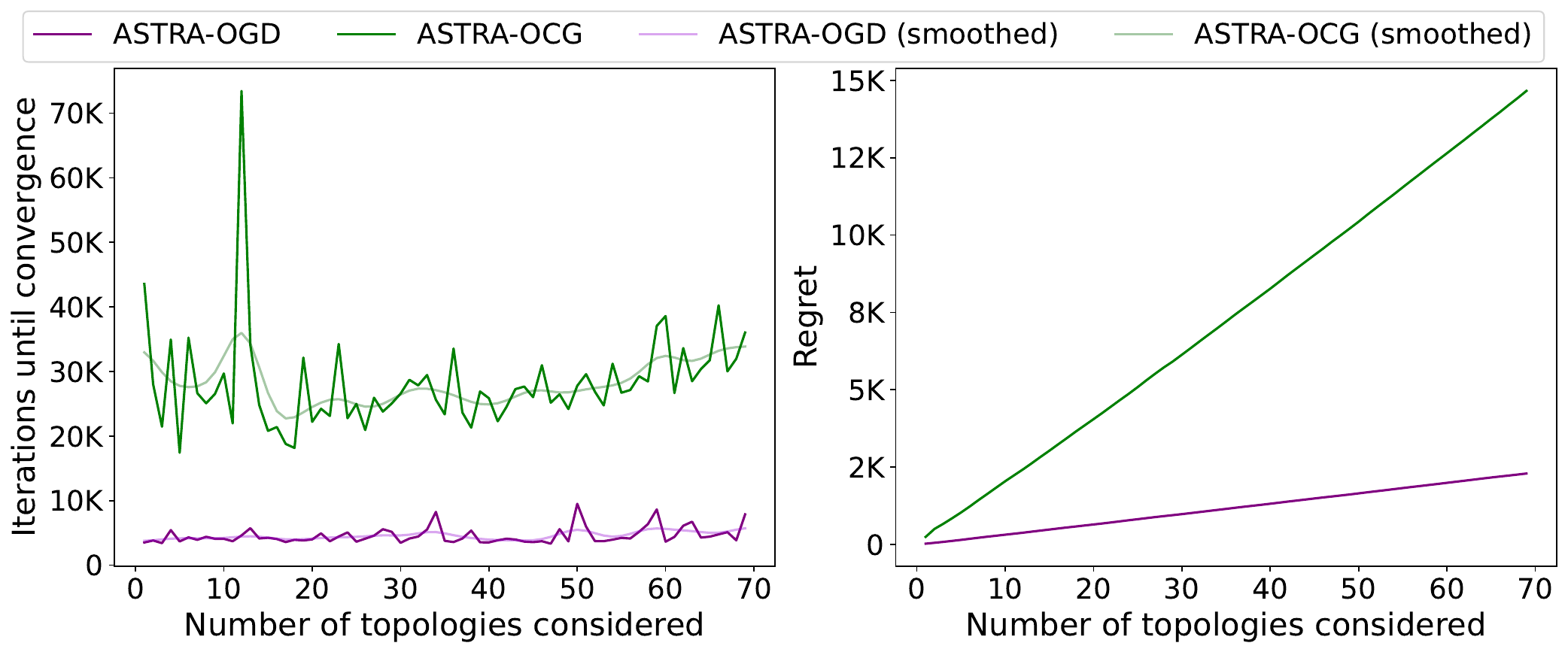}
    \caption[Number of iterations until convergence and regret of ASTRA-OGD and ASTRA-OCG on Starlink data.]{Left - Number of iterations until convergence for ASTRA-OGD and ASTRA-OCG on Starlink data. Right - Regret evolution of both algorithms on Starlink data. ASTRA-OGD obtains the lowest regret, taking significantly fewer iterations to converge.}
    \label{fig:astra-ogd_vs_astra-ocg_its}
\end{figure}

%% file: Tables/astra-ocg_iterations_time.tex
\begin{table}[]
\centering
\scriptsize
\caption[Wall-clock time needed for ASTRA-OCG to achieve a certain number of iterations]{Wall-clock time needed for ASTRA-OCG to achieve a certain number of iterations for a fixed topology. The mean values and respective std are presented for 10 runs.}
\begin{tabular}{@{}cccc@{}}
\toprule
Number of iterations: & 20K                     & 30k         & 40k        \\ \midrule
Time (s)              & \(330.870_{\pm 3.462}\) & \(430.103_{\pm 0.983}\) & \(658.597_{\pm 3.747}\)  \\ \bottomrule
\end{tabular}
\label{tab:astra-ocg-time-its}
\end{table}

%% file: Tables/learning-based_comparison.tex
\begin{table}[t]
    \centering
    \scriptsize
    \caption{Comparison of GUTO and ASTRA with state-of-the-art non-learning and learning-based methods across 76 Iridium-like constellations.}
    \label{tab:learning-based-comp}
    \begin{tabular}{ccccc}
        \toprule
        Method: & E \(\uparrow\) & A-Deg\(\uparrow\) & CC \(\downarrow\) & A-SP \(\downarrow\)\\
        \midrule
        MAIILP-0.2  & \(90.474 \pm 2.292\) & \(2.742 \pm 0.058\) & \(1 \pm 0.0\)     & \(5.201 \pm 0.178\) \\
        MAIILP-1e-3 & \(93.276 \pm 2.292\) & \(2.827 \pm 0.069\) & \(1 \pm 0.0\)     & \(5.002 \pm 0.127\) \\
        MAIILP-1e-4 & \(92.842 \pm 1.993\) & \(2.813 \pm 0.060\) & \(1 \pm 0.0\)     & \(5.038 \pm 0.113\) \\
        \midrule
        +Grid       & \(130 \pm 0.000\)   & \(3.939 \pm 0.001\) & \(1 \pm 0.0\)     & \(4.036 \pm 0.001\) \\
        xGrid       & \(124 \pm 0.000\)   & \(3.757 \pm 0.001\) & \(1 \pm 0.0\)     & \(4.978 \pm 0.001\) \\
        DoTD        & \(132 \pm 0.388\)   & \(3.973 \pm 0.010\) & \(1.5 \pm 0.5\) & \(4.351 \pm 0.025\) \\
        GUTO        & \(128 \pm 0.000\)   & \(3.880 \pm 0.001\) & \(1 \pm 0.0\)     & \(3.767 \pm 0.001\) \\
        ASTRA       & \(128 \pm 0.000\)   & \(3.880 \pm 0.001\) & \(1 \pm 0.0\)     & \(3.767 \pm 0.001\) \\
        \bottomrule
    \end{tabular}
\end{table}

%% file: Figures_tex/mip_regret_comparisons.tex
\begin{figure}[t]
    \centering
    \scriptsize
    \includegraphics[width=.75\columnwidth]{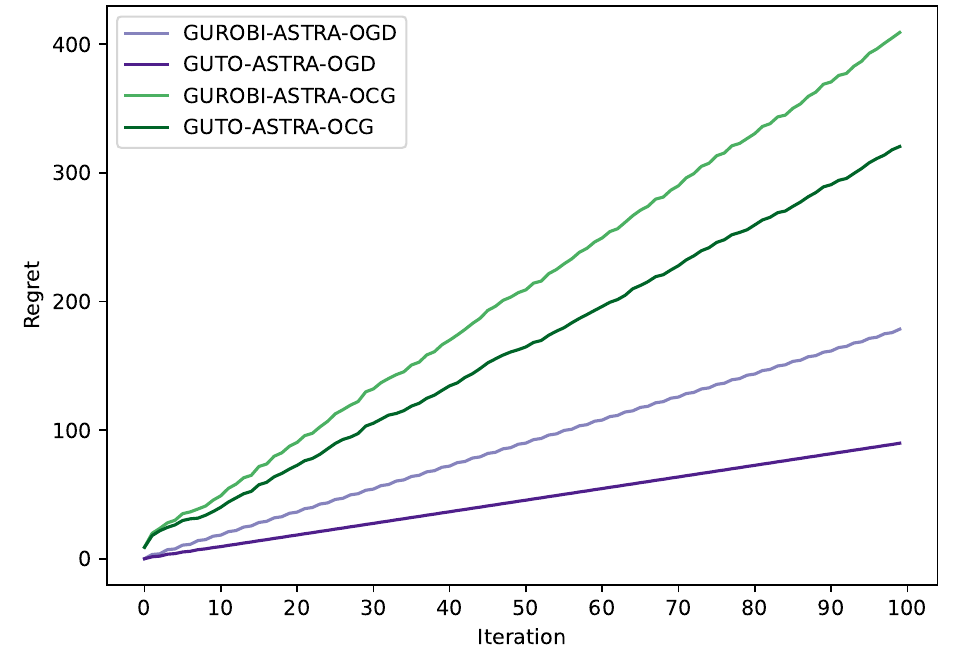}
	\caption{Regret comparison between ASTRA online algorithms, the optimal solution proposed by GUROBI and GUTO.}
	\label{fig:mip_regret}
\end{figure}

%% file: mybib.bib
@book{axler2024linear,
  title={Linear algebra done right},
  author={Axler, Sheldon},
  year={2024},
  publisher={Springer},
  doi={doi.org/10.1007/978-3-031-41026-0}
}

@article{hertzfeld2004weather,
  title={Weather satellites and the economic value of forecasts: evidence from the electric power industry},
  author={Hertzfeld, Henry R and Williamson, Ray A and Sen, Avery},
  journal={Acta Astronautica},
  volume={55},
  number={3-9},
  pages={791--802},
  year={2004},
  publisher={Elsevier}
}

@article{shaengchart2023starlink,
  title={Starlink satellite project impact on the Internet provider service in emerging economies},
  author={Shaengchart, Yarnaphat and Kraiwanit, Tanpat},
  journal={Research in Globalization},
  volume={6},
  pages={100132},
  year={2023},
  publisher={Elsevier}
}

@article{de2015satellite,
  title={Satellite communications supporting internet of remote things},
  author={De Sanctis, Mauro and Cianca, Ernestina and Araniti, Giuseppe and Bisio, Igor and Prasad, Ramjee},
  journal={IEEE Internet of Things Journal},
  volume={3},
  number={1},
  pages={113--123},
  year={2015},
  publisher={IEEE}
}

@article{chen2024concept,
  title={The concept, technical architecture, applications and impacts of satellite internet: A systematic literature review},
  author={Chen, Yan and Ma, Xin and Wu, Chaonan},
  journal={Heliyon},
  volume={10},
  number={13},
  year={2024},
  publisher={Elsevier}
}

@ARTICLE{9461407,
  author={Lee, Yonghwa and Choi, Jihwan P.},
  journal={IEEE Transactions on Aerospace and Electronic Systems}, 
  title={Connectivity Analysis of Mega-Constellation Satellite Networks With Optical Intersatellite Links}, 
  year={2021},
  volume={57},
  number={6},
  pages={4213-4226},
  doi={10.1109/TAES.2021.3090914}
}

@inproceedings{kulu2024satellite,
  title={Satellite constellations—2024 survey, trends and economic sustainability},
  author={Kulu, Erik},
  booktitle={Proceedings of the International Astronautical Congress, IAC, Milan, Italy},
  pages={14--18},
  year={2024}
}

@article{gonzalo2014challenge,
  title={On the challenge of a century lifespan satellite},
  author={Gonzalo, Jes{\'u}s and Dom{\'\i}nguez, Diego and L{\'o}pez, Deibi},
  journal={Progress in Aerospace Sciences},
  volume={70},
  pages={28--41},
  year={2014},
  publisher={Elsevier}
}

@article{guo2025resilience,
  title={Resilience of Mega-Satellite Constellations: How Node Failures Impact Inter-Satellite Networking Over Time?},
  author={Guo, Binquan and Xiong, Zehui and Zhang, Zhou and Li, Baosheng and Niyato, Dusit and Yuen, Chau and Han, Zhu},
  journal={IEEE Transactions on Communications},
  year={2025},
  publisher={IEEE}
}

@article{xiaogang2016survey,
  title={A survey of routing techniques for satellite networks},
  author={Xiaogang, QI and Jiulong, Ma and Dan, Wu and Lifang, Liu and Shaolin, Hu},
  journal={Journal of communications and information networks},
  volume={1},
  number={4},
  pages={66--85},
  year={2016},
  publisher={PTP}
}

@inproceedings{sun2020improved,
  title={An Improved Routing Strategy Based on Virtual Topology in LEO Satellite Networks},
  author={Sun, Chaoran and Zhang, Yu and Zhu, Jian},
  booktitle={International Conference on Wireless and Satellite Systems},
  pages={240--250},
  year={2020},
  organization={Springer}
}

@article{lu2013virtual,
  title={Virtual topology for LEO satellite networks based on earth-fixed footprint mode},
  author={Lu, Yong and Sun, Fuchun and Zhao, Youjian},
  journal={IEEE communications letters},
  volume={17},
  number={2},
  pages={357--360},
  year={2013},
  publisher={IEEE}
}

@INPROCEEDINGS{10404740,
  author={Dong, Yunyao and Xu, Xiaofan and Zhang, Yueyue and Wang, Siming and Du, Ping and Xu, Du and Zhang, Xiaoning},
  booktitle={2023 International Conference on Wireless Communications and Signal Processing (WCSP)}, 
  title={A Novel Virtual Node-Based Multi-Controller Management Architecture for LEO Mega-Constellation Satellite Networks}, 
  year={2023},
  volume={},
  number={},
  pages={701-706},
  doi={10.1109/WCSP58612.2023.10404740}
}

@article{ekici2002distributed,
  title={A distributed routing algorithm for datagram traffic in LEO satellite networks},
  author={Ekici, Eylem and Akyildiz, Ian F and Bender, Michael D},
  journal={IEEE/ACM Transactions on networking},
  volume={9},
  number={2},
  pages={137--147},
  year={2002},
  publisher={IEEE}
}

@article{10436098,
  author={Lyu, Yifeng and Hu, Han and Fan, Rongfei and Liu, Zhi and An, Jianping and Mao, Shiwen},
  journal={IEEE Journal on Selected Areas in Communications}, 
  title={Dynamic Routing for Integrated Satellite-Terrestrial Networks: A Constrained Multi-Agent Reinforcement Learning Approach}, 
  year={2024},
  volume={42},
  number={5},
  pages={1204-1218},
  doi={10.1109/JSAC.2024.3365869}
}

@article{papapetrou2007distributed,
  title={Distributed on-demand routing for {LEO} satellite systems},
  author={Papapetrou, Evangelos and Karapantazis, Stylianos and Pavlidou, F-N},
  journal={Computer networks},
  volume={51},
  number={15},
  pages={4356--4376},
  year={2007},
  publisher={Elsevier}
}

@article{leyva2021inter,
  title={Inter-plane inter-satellite connectivity in dense LEO constellations},
  author={Leyva-Mayorga, Israel and Soret, Beatriz and Popovski, Petar},
  journal={IEEE Transactions on Wireless Communications},
  volume={20},
  number={6},
  pages={3430--3443},
  year={2021},
  publisher={IEEE}
}

@inproceedings{ron2025time,
  title={Time-Dependent Network Topology Optimization for LEO Satellite Constellations},
  author={Ron, Dara and Yusufzai, Faisal Ahmed and Kwakye, Sebastian and Roy, Satyaki and Sastry, Nishanth and Shah, Vijay K},
  booktitle={IEEE INFOCOM 2025-IEEE Conference on Computer Communications},
  pages={1--10},
  year={2025},
  organization={IEEE}
}

@inproceedings{mclaughlin2023grid,
  title={$\times$ grid: A location-oriented topology design for LEO satellites},
  author={Mclaughlin, Joseph and Choi, Jee and Durairajan, Ramakrishnan},
  booktitle={Proceedings of the 1st ACM Workshop on LEO Networking and Communication},
  pages={37--42},
  year={2023}
}

@inproceedings{Network-topology-design-at-27k-km/hour,
author = {Bhattacherjee, Debopam and Singla, Ankit},
title = {Network topology design at 27,000 km/hour},
year = {2019},
isbn = {9781450369985},
publisher = {Association for Computing Machinery},
address = {New York, NY, USA},
doi = {10.1145/3359989.3365407},
booktitle = {Proceedings of the 15th International Conference on Emerging Networking Experiments And Technologies},
pages = {341–354},
numpages = {14},
location = {Orlando, Florida},
series = {CoNEXT '19}
}

@article{hazan2016introduction,
  title={Introduction to online convex optimization},
  author={Hazan, Elad and others},
  journal={Foundations and Trends{\textregistered} in Optimization},
  volume={2},
  number={3-4},
  pages={157--325},
  year={2016},
  publisher={Now Publishers, Inc.}
}

@inproceedings{han2021dynamic,
  title={Dynamic routing for software-defined LEO satellite networks based on ISL attributes},
  author={Han, Zhenzhen and Zhao, Guofeng and Xing, Yuan and Sun, Nanbin and Xu, Chuan and Yu, Shui},
  booktitle={2021 IEEE global communications conference (GLOBECOM)},
  pages={1--6},
  year={2021},
  organization={IEEE}
}

@article{8688478,
    author={Li, Jian and Lu, Hancheng and Xue, Kaiping and Zhang, Yongdong},
    journal={IEEE Transactions on Vehicular Technology},
    title={Temporal Netgrid Model-Based Dynamic Routing in Large-Scale Small Satellite Networks},
    year={2019},
    volume={68},
    number={6},
    pages={6009-6021},
    doi={10.1109/TVT.2019.2910570}
}

@ARTICLE{taleb2009,
  author={Taleb, Tarik and Mashimo, Daisuke and Jamalipour, Abbas and Kato, Nei and Nemoto, Yoshiaki},
  journal={IEEE/ACM Transactions on Networking}, 
  title={Explicit Load Balancing Technique for NGEO Satellite IP Networks With On-Board Processing Capabilities}, 
  year={2009},
  volume={17},
  number={1},
  pages={281-293},
  doi={10.1109/TNET.2008.918084}
}

@article{ali2026leo,
  title={LEO Topology Design Under Real-World Deployment Constraints},
  author={Ali, Muaz and Zhang, Beichuan},
  journal={arXiv preprint arXiv:2602.07756},
  year={2026}
}

@article{wu2025enhancing,
  title={Enhancing LEO mega-constellations with inter-satellite links: Vision and challenges},
  author={Wu, Chenyu and Han, Shuai and Chen, Qian and Wang, Yu and Meng, Weixiao and Benslimane, Abderrahim},
  journal={IEEE Wireless Communications},
  year={2025},
  publisher={IEEE}
}

@inproceedings{bhattacharjee2023laser,
  title={Laser inter-satellite link setup delay: Quantification, impact, and tolerable value},
  author={Bhattacharjee, Dhiraj and Chaudhry, Aizaz U and Yanikomeroglu, Halim and Hu, Peng and Lamontagne, Guillaume},
  booktitle={2023 IEEE Wireless Communications and Networking Conference (WCNC)},
  pages={1--6},
  year={2023},
  organization={IEEE}
}

@article{norberto2026online,
  title={Online Learning for Dynamic Constellation Topologies},
  author={Norberto, Jo{\~a}o and Ferreira, Ricardo and Soares, Cl{\'a}udia},
  journal={arXiv preprint arXiv:2603.25954},
  year={2026}
}

@inproceedings{pi2022dynamic,
  title={Dynamic planning of inter-plane inter-satellite links in LEO satellite networks},
  author={Pi, Jiahao and Ran, Yongyi and Wang, Hao and Zhao, Yanyun and Zhao, Ruili and Luo, Jiangtao},
  booktitle={ICC 2022-IEEE International Conference on Communications},
  pages={3070--3075},
  year={2022},
  organization={IEEE}
}

@article{burke1996hoffman,
author = {Burke, James V. and Tseng, Paul},
title = {A Unified Analysis of Hoffman’s Bound via Fenchel Duality},
journal = {SIAM Journal on Optimization},
volume = {6},
number = {2},
pages = {265-282},
year = {1996},
doi = {10.1137/0806015}
}

@INPROCEEDINGS{EssweinLAMP,
  author={Esswein, Robert and Bayer, Quincy and Mergendahl, Samuel and Ruffley, Jon and Abdelhakim, Mai and Cunningham, Robert},
  booktitle={ICC 2025 - IEEE International Conference on Communications}, 
  title={LAMP: Low-Latency Dynamic Topology for LEO Satellite Constellations}, 
  year={2025},
  volume={},
  number={},
  pages={4714-4719},
  doi={10.1109/ICC52391.2025.11162074}
}

@misc{gurobi,
  author = {{Gurobi Optimization, LLC}},
  title = {{Gurobi Optimizer Reference Manual}},
  year = 2026,
  url = "https://www.gurobi.com"
}

@book{cesa2006prediction,
  title={Prediction, learning, and games},
  author={Cesa-Bianchi, Nicolo and Lugosi, G{\'a}bor},
  year={2006},
  publisher={Cambridge university press},
  doi={10.1017/CBO9780511546921}
}

@InProceedings{hazan2011stochastic,
  title={Beyond the regret minimization barrier: an optimal algorithm for stochastic strongly-convex optimization},
  author={Hazan, Elad and Kale, Satyen},
  booktitle={Proceedings of the 24th Annual Conference on Learning Theory},
  pages={421--436},
  year={2011},
  volume={19},
  series={Proceedings of Machine Learning Research},
  publisher={PMLR},
}

@article{nonhoff2026control,
title = {Online convex optimization for constrained control of nonlinear systems},
journal = {Automatica},
volume = {188},
pages = {112916},
year = {2026},
issn = {0005-1098},
doi = {doi.org/10.1016/j.automatica.2026.112916},
author = {Marko Nonhoff and Johannes Köhler and Matthias A. Müller}
}

@article{cover1991portfolio,
author = {Cover, Thomas M.},
title = {Universal Portfolios},
journal = {Mathematical Finance},
volume = {1},
number = {1},
pages = {1-29},
doi = {doi.org/10.1111/j.1467-9965.1991.tb00002.x},
year = {1991}
}

@inproceedings{zhang2018adaptive,
 author = {Zhang, Lijun and Lu, Shiyin and Zhou, Zhi-Hua},
 booktitle = {Advances in Neural Information Processing Systems},
 editor = {S. Bengio and H. Wallach and H. Larochelle and K. Grauman and N. Cesa-Bianchi and R. Garnett},
 publisher = {Curran Associates, Inc.},
 title = {Adaptive Online Learning in Dynamic Environments},
 url = {https://proceedings.neurips.cc/paper_files/paper/2018/file/10a5ab2db37feedfdeaab192ead4ac0e-Paper.pdf},
 volume = {31},
 year = {2018}
}

@inproceedings{zhao2020dynamic,
 author = {Zhao, Peng and Zhang, Yu-Jie and Zhang, Lijun and Zhou, Zhi-Hua},
 booktitle = {Advances in Neural Information Processing Systems},
 editor = {H. Larochelle and M. Ranzato and R. Hadsell and M.F. Balcan and H. Lin},
 pages = {12510--12520},
 publisher = {Curran Associates, Inc.},
 title = {Dynamic Regret of Convex and Smooth Functions},
 url = {https://proceedings.neurips.cc/paper_files/paper/2020/file/939314105ce8701e67489642ef4d49e8-Paper.pdf},
 volume = {33},
 year = {2020}
}

@inproceedings{cassel2022control,
 author = {Cassel, Asaf Benjamin and Peled-Cohen, Alon and Koren, Tomer},
 booktitle = {Advances in Neural Information Processing Systems},
 editor = {S. Koyejo and S. Mohamed and A. Agarwal and D. Belgrave and K. Cho and A. Oh},
 pages = {7410--7422},
 publisher = {Curran Associates, Inc.},
 title = {Rate-Optimal Online Convex Optimization in Adaptive Linear Control},
 url = {https://proceedings.neurips.cc/paper_files/paper/2022/file/30dfe47a3ccbee68cffa0c19ccb1bc00-Paper-Conference.pdf},
 volume = {35},
 year = {2022}
}

@inproceedings{gatmiry2023projection,
 author = {Gatmiry, Khashayar and Mhammedi, Zak},
 booktitle = {Advances in Neural Information Processing Systems},
 editor = {A. Oh and T. Naumann and A. Globerson and K. Saenko and M. Hardt and S. Levine},
 pages = {986--1008},
 publisher = {Curran Associates, Inc.},
 title = {Projection-Free Online Convex Optimization via Efficient Newton Iterations},
 url = {https://proceedings.neurips.cc/paper_files/paper/2023/file/03261886741f1f21f52f2a2d570616a2-Paper-Conference.pdf},
 volume = {36},
 year = {2023}
}

@inproceedings{schmidt2011inexact,
 author = {Schmidt, Mark and Roux, Nicolas and Bach, Francis},
 booktitle = {Advances in Neural Information Processing Systems},
 editor = {J. Shawe-Taylor and R. Zemel and P. Bartlett and F. Pereira and K. Weinberger},
 publisher = {Curran Associates, Inc.},
 title = {Convergence Rates of Inexact Proximal-Gradient Methods for Convex Optimization},
 url = {https://proceedings.neurips.cc/paper_files/paper/2011/file/8f7d807e1f53eff5f9efbe5cb81090fb-Paper.pdf},
 volume = {24},
 year = {2011}
}

@article{liang2016convergence,
  title={Convergence rates with inexact non-expansive operators},
  author={Liang, Jingwei and Fadili, Jalal and Peyr{\'e}, Gabriel},
  journal={Mathematical Programming},
  volume={159},
  number={1},
  pages={403--434},
  year={2016},
  publisher={Springer},
  doi={10.1007/s10107-015-0964-4}
}

@ARTICLE{dixit2019inexact,
  author={Dixit, Rishabh and Bedi, Amrit Singh and Tripathi, Ruchi and Rajawat, Ketan},
  journal={IEEE Transactions on Signal Processing}, 
  title={Online Learning With Inexact Proximal Online Gradient Descent Algorithms}, 
  year={2019},
  volume={67},
  number={5},
  pages={1338-1352},
  doi={10.1109/TSP.2018.2890368}
}

@article{hazan2007logarithmic,
  title={Logarithmic regret algorithms for online convex optimization},
  author={Hazan, Elad and Agarwal, Amit and Kale, Satyen},
  journal={Machine Learning},
  volume={69},
  number={2},
  pages={169--192},
  year={2007},
  publisher={Springer},
  doi={10.1007/s10994-007-5016-8}
}

@InProceedings{gradu2023control,
  title={Adaptive Regret for Control of Time-Varying Dynamics},
  author={Gradu, Paula and Hazan, Elad and Minasyan, Edgar},
  booktitle={Proceedings of The 5th Annual Learning for Dynamics and Control Conference},
  pages={560--572},
  year={2023},
  volume={211},
  series={Proceedings of Machine Learning Research},
  publisher={PMLR},
}

@inproceedings{zinkevich2003online,
  title={Online convex programming and generalized infinitesimal gradient ascent},
  author={Zinkevich, Martin},
  booktitle={Proceedings of the 20th international conference on machine learning (icml-03)},
  pages={928--936},
  year={2003}
}

@article{duchi2011adaptive,
author = {Duchi, John and Hazan, Elad and Singer, Yoram},
title = {Adaptive Subgradient Methods for Online Learning and Stochastic Optimization},
year = {2011},
issue_date = {2/1/2011},
publisher = {JMLR.org},
volume = {12},
number = {null},
issn = {1532-4435},
journal = {Journal of Machine Learning Research},
month = jul,
pages = {2121–2159},
numpages = {39}
}

@inproceedings{hazan2012projection,
author = {Hazan, Elad and Kale, Satyen},
title = {Projection-free online learning},
year = {2012},
isbn = {9781450312851},
publisher = {Omnipress},
address = {Madison, WI, USA},
booktitle = {Proceedings of the 29th International Coference on International Conference on Machine Learning},
pages = {1843–1850},
numpages = {8},
location = {Edinburgh, Scotland},
series = {ICML'12}
}

@inproceedings{hazan2020faster,
  title={Faster projection-free online learning},
  author={Hazan, Elad and Minasyan, Edgar},
  booktitle={Conference on Learning Theory},
  pages={1877--1893},
  year={2020},
  organization={PMLR}
}

@inproceedings{hazan2021boosting,
  title={Boosting for online convex optimization},
  author={Hazan, Elad and Singh, Karan},
  booktitle={International Conference on Machine Learning},
  pages={4140--4149},
  year={2021},
  organization={PMLR}
}

@article{bubeck2012regret,
  title={Regret analysis of stochastic and nonstochastic multi-armed bandit problems},
  author={Bubeck, S{\'e}bastien and Cesa-Bianchi, Nicolo},
  journal={Foundations and Trends{\textregistered} in Machine Learning},
  volume={5},
  number={1},
  pages={1--122},
  year={2012},
  publisher={Now Publishers, Inc.}
}

@book{HiriartUrruty1993,
  title = {Convex Analysis and Minimization Algorithms I},
  ISBN = {9783662027967},
  ISSN = {0072-7830},
  DOI = {10.1007/978-3-662-02796-7},
  journal = {Grundlehren der mathematischen Wissenschaften},
  publisher = {Springer Berlin Heidelberg},
  author = {Hiriart-Urruty,  Jean-Baptiste and Lemaréchal,  Claude},
  year = {1993}
}

@article{orabona2019modern,
  title={A modern introduction to online learning},
  author={Orabona, Francesco},
  journal={arXiv preprint arXiv:1912.13213},
  year={2019}
}

@book{boyd2004convex, 
    place={Cambridge}, 
    title={Convex Optimization}, 
    publisher={Cambridge University Press}, 
    author={Boyd, Stephen and Vandenberghe, Lieven}, 
    year={2004}
}

@book{horn2012matrix,
  title={Matrix analysis},
  author={Horn, Roger A and Johnson, Charles R},
  year={2012},
  publisher={Cambridge university press}
}

@article{boyd2011admm,
  title={Distributed optimization and statistical learning via the alternating direction method of multipliers},
  author={Boyd, Stephen and Parikh, Neal and Chu, Eric and Peleato, Borja and Eckstein, Jonathan},
  journal={Foundations and Trends{\textregistered} in Machine learning},
  volume={3},
  number={1},
  pages={1--122},
  year={2011},
  publisher={Emerald Publishing Limited}
}
